\documentclass{article}

\usepackage{microtype}
\usepackage{graphicx}
\graphicspath{ {../figures/} }
\usepackage{booktabs} %
\usepackage{hyperref}

\usepackage{amsmath}
\usepackage{amsthm}
\usepackage{amssymb}
\usepackage{bbm} 
\usepackage{bm}
\usepackage{verbatim}
\usepackage{float}
\usepackage{color,soul}
\usepackage{enumitem}
\usepackage{mathtools}
\usepackage{hhline}
\usepackage[title]{appendix}
\usepackage[nameinlink]{cleveref} %
\usepackage[font=small,labelfont=bf,
   justification=justified,
   format=plain]{caption}
\usepackage{subcaption}
\usepackage{tabularx}
\usepackage{tikz}
\usepackage{wrapfig,booktabs}
\usepackage{xspace}
\usepackage{cancel}
\usepackage{sidecap}

\DeclarePairedDelimiterX{\infdivx}[2]{(}{)}{%
  #1\;\delimsize\|\;#2%
}

\newcommand{\dgps}{\textsc{dgp}s\xspace}
\newcommand{\dgp}{\textsc{dgp}\xspace}
\newcommand{\gps}{\textsc{gp}s\xspace}
\newcommand{\gp}{\textsc{gp}\xspace}

\newcommand{\elbo}{\textsc{elbo}\xspace}
\newcommand{\elbos}{\textsc{elbo}s\xspace}

\crefname{appsec}{appendix}{appendices}
\Crefname{appsec}{Appendix}{Appendices}

\newtheorem{remark}{Remark}
\newtheorem{theorem}{Theorem}

\newtheorem{innercustomcondition}{Condition}
\newenvironment{customcondition}[1]
  {\renewcommand\theinnercustomcondition{#1}\innercustomcondition}
  {\endinnercustomcondition}

\usetikzlibrary{shapes.geometric, arrows, bayesnet, calc, positioning}

\newcommand{\calF}{\mathcal{F}}

\newcommand{\calGP}{\mathcal{GP}}

\newcommand{\calN}{\mathcal{N}}
\newcommand{\calO}{\mathcal{O}}
\newcommand{\calL}{\mathcal{L}}

\newcommand{\Zhat}{\hat{Z}}
\newcommand{\ztilde}{\tilde{z}}
\newcommand{\ftilde}{\tilde{f}}

\newcommand{\closer}[3]{{\kern-#1ex{#2}\kern-#3ex}}

\newcommand{\DKL}{D_{\text{KL}}\infdivx}

\DeclareMathOperator{\gradphi}{\nabla_{\phi}}

\mathchardef\mhyphen="2D

\DeclareMathOperator{\E}{\mathbb{E}}

\newcommand\defines{\stackrel{\mathclap{\normalfont\mbox{\tiny def}}}{=}}

\newcommand{\vae}{\textsc{vae}\xspace}

\newcommand{\iwvae}{\textsc{iwvae}\xspace}

\newcommand{\iwae}{\textsc{iwae}\xspace}
\newcommand{\iwaes}{\textsc{iwae}s\xspace}
\newcommand{\svi}{\textsc{svi}\xspace}

\newcommand{\iwvi}{\textsc{iwvi}\xspace}
\newcommand{\snr}{\textsc{snr}\xspace}
\newcommand{\snrs}{\textsc{snr}s\xspace}
\newcommand{\reg}{\textsc{reg}\xspace}
\newcommand{\dreg}{\textsc{dreg}\xspace}

\definecolor{dreg}{rgb}{1.0,0.498039215686274, 0.054901960784313725}
\definecolor{reg}{rgb}{0.12156862745098039,0.4666666666666667,0.7058823529411765}
\definecolor{highlight}{rgb}{0.17254901960784313, 0.6274509803921569, 0.17254901960784313}

\hypersetup{ %
    colorlinks=true,
    linkcolor=black,
    citecolor=black,
    filecolor=black,
    urlcolor=black,
}

\usepackage{hyperref}

\usepackage[accepted]{include/icml2021}

\icmltitlerunning{
    On Signal-to-Noise Ratio Issues in Variational Inference for Deep Gaussian Processes
}

\begin{document}

\twocolumn[
\icmltitle{
    On Signal-to-Noise Ratio Issues in
    \\
    Variational Inference for Deep Gaussian Processes
}

\icmlsetsymbol{equal}{*}

\begin{icmlauthorlist}
\icmlauthor{Tim G. J. Rudner}{equal,oxford}
\icmlauthor{Oscar Key}{equal,ucl}
\icmlauthor{Yarin Gal}{oxford}
\icmlauthor{Tom Rainforth}{oxford}
\end{icmlauthorlist}

\icmlaffiliation{oxford}{Department of Computer Science, University of Oxford, Oxford, United Kingdon}
\icmlaffiliation{ucl}{Computer Science Department, University College London, London, United Kingdom}

\icmlcorrespondingauthor{Tim G. J. Rudner}{tim.rudner@cs.ox.ac.uk}

\vskip 0.3in
]

\printAffiliationsAndNotice{\icmlEqualContribution} %

\begin{abstract}
We show that the gradient estimates used in training Deep Gaussian Processes (\dgps) with importance-weighted variational inference are susceptible to signal-to-noise ratio (\snr) issues. Specifically, we show both theoretically and via an extensive empirical evaluation that the \snr of the gradient estimates for the latent variable's variational parameters \emph{decreases} as the number of importance samples \emph{increases}. As a result, these gradient estimates degrade to pure noise if the number of importance samples is too large. To address this pathology, we show how doubly reparameterized gradient estimators, originally proposed for training variational autoencoders, can be adapted to the \dgp setting and that the resultant estimators completely remedy the \snr issue, thereby providing more reliable training. Finally, we demonstrate that our fix can lead to consistent improvements in the predictive performance of \dgp models.
\end{abstract}

\section{Introduction}
\label{sec:introduction}

Deep Gaussian Processes (\dgps) are a powerful class of probabilistic models for supervised learning tasks~\citep{lawrence2013dgps,bui2016dgps,salimbeni2017dsvi}.
They are a multi-layer hierarchical generalization of conventional Gaussian processes (\gps,~\citet{ramussen2005gpbook}), themselves flexible non-parametric models that have seen a wide range of applications to a variety of machine learning problems~\citep{mockus1994bo,hensman15, wilson2015kernel}.
\dgps aim to retain the advantages of \gps---such as their well-calibrated uncertainty estimates and robustness to over-fitting---while also overcoming their limitations---such as the restricted form of their predictive distribution.
In essence, \dgps are an even more general and powerful class of models than traditional \gps~\citep{lawrence2013dgps}, allowing us to combine multi-layer function composition with a principled Bayesian approach to obtaining predictive uncertainties.

\begin{figure}[t!]
    \center
    \includegraphics[width=0.87\linewidth, keepaspectratio]{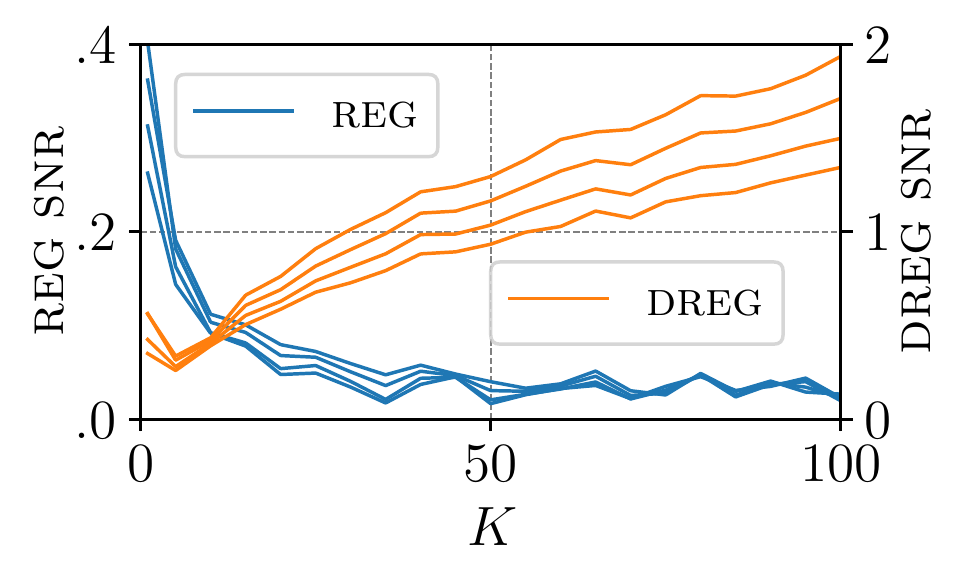}
    \caption{
    Signal-to-noise ratio (\snrs) in \iwvi for \dgps under the standard gradient estimators (\reg) and our improved estimator (\dreg) as a function of the number of importance samples ($K$).
    The different lines show \snrs for four different variational parameters.
    As $K$ increases, the \snr of \reg decreases and the \snr of \dreg increases.
    }
    \label{fig:snr_reg_example}
    \vspace*{-13pt}
\end{figure}

Unfortunately, this modeling power comes with a caveat: Exact inference in \dgps is intractable and one has to rely on approximate inference schemes to train them.
Typically this is done using stochastic variational inference (\svi) methods~\citep{hoffman2013svi,salimbeni2017dsvi}, which recast the inference problem into that of stochastic optimization of an evidence lower bound (\elbo) with respect to the parameters of an approximating distribution.
The resulting \svi-based approaches to training \dgps lead to state-of-the-art performance on a wide range of challenging problems requiring uncertainty quantification~\citep{salimbeni2017dsvi}.

Despite the successes of these approaches,~\citet{salimbeni2019iwvi} find that they do not perform well when applied to \dgp models that have been extended to include latent variables that allow them to represent non-Gaussianity and multimodality in data.
To address this, they propose an \textit{importance-weighted variational inference} (\iwvi;~\citet{burda2015iwae,mnih2016variational,domke2018iwae}) approach that provides both a tighter variational bound and lower-variance gradient updates by importance sampling the latent variables.
They show that this in turn leads to learning improved predictive distributions.

In this paper, we highlight a potential shortcoming of such approaches:
Tightening the bound can cause a deterioration in the \emph{signal-to-noise ratio} (\snr) of the gradient estimates associated with the variational parameters of the latent variables (see Figure~\ref{fig:snr_reg_example} for a demonstration).
Focusing on the particular example of \iwvi, we see that even though using more importance samples both tightens the bound and reduces the variance of the gradients, it can actually \emph{increase} the \emph{relative variance} of the gradients estimates compared to the true gradient, which is itself also tending to zero as we use more samples.
This weaker gradient signal leads to difficulties optimizing the variational parameters, a phenomenon analogous to known issues in training the inference networks of variational autoencoders~\citep{rainforth2018bounds}.
We demonstrate this \snr degradation both theoretically in the limit of using a large number of samples, and empirically for samples sizes often used in practice.

To address this degradation, we show how the \emph{doubly reparameterized gradient} (\dreg) estimator of~\citet{tucker2019dreg} can be adapted to the \dgp \iwvi bound.
The theoretical and empirical \snr of the resulting gradient estimator increases for \emph{all} parameters as the number of importance samples is increased, thereby providing more accurate gradient updates in the training of \dgps with latent variables.
It can be used as a drop-in replacement for the standard estimator at no computational or other cost.
Empirically, we find that this fix can lead to improvements in predictive performance.
In particular, we show that when the quality of latent variable variational approximation is important for prediction, then our estimator typically improves performance.\footnote{Our code is available at \url{https://github.com/timrudner/snr_issues_in_deep_gps}.}

To summarize, our core contributions are as follows:\vspace{-5pt}
\begin{itemize}[leftmargin=20pt]
\setlength\itemsep{0.1em}
    \item We highlight the presence of a signal-to-noise ratio (\snr) issue in the training of \dgps with \iwvi as the number of importance samples increases;
    \item We nullify this \snr issue by introducing a doubly reparameterized gradient (\dreg) estimator;
    \item We quantify the detrimental effect of this \snr deterioration on predictive performance, and show that our fix leads to predictive distributions conferring statistically significant improvements.
\end{itemize}

\newpage

\section{Preliminaries \& Related Work}
\label{sec:background}

\subsection{Latent Deep Gaussian Process Models}
\label{sec:dgp_model}

Let $X = [x_1, ..., x_N]^\top$ denote a collection of $N$ $D$-dimensional input data points and let $Y = [y_1, ..., y_N]^\top$ denote a collection of $N$ $P$-dimensional noisy observations.
A \dgp model is defined as the composition
\begin{align*}
	y_{n} = f^{(L)}(f^{(L-1)}(...f^{(1)}(x_{n}))...) + \varepsilon_{n},
\end{align*}
where $\varepsilon_{n} \sim \calN(0, \sigma^2 I_{P} )$, $L$ is the number of layers,
and each $f^{(\ell)}(f^{(\ell-1)})$ in the composition denotes a \gp evaluated at the draws from the previous layer, $f^{(\ell - 1)}$.
As per~\cite{salimbeni2019iwvi}, we further augment this model with a $\tilde{D}$-dimensional latent variable $z$ which extends the input space.
The resulting \textit{latent-variable} \dgp model is given by
\begin{align*}
	y_{n} = f^{(L)}(f^{(L-1)}(...f^{(1)}([x_{n}, z_{n}]))...) + \varepsilon_{n},
\end{align*}
where $\varepsilon_{n} \sim \calN(0, \sigma^2 I_{P} )$ and $z_{n} \sim \calN(0, I_{\tilde{D}})$.
For simplicity of exposition, our notation here and throughout the paper assumes a \dgp with two \gp layers, $f^{(1)}$ and $f^{(2)}$, but emphasize that our results apply to \dgps of \emph{any depth} (including latent shallow \gps) and we include experiments that use more than two layers.
Assuming a Gaussian likelihood with scalar noise $\sigma^2$,
\begin{align*}
    p(y_{n} \,|\, f^{(1)}&, f^{(2)}, z_{n}; x_{n}) \\
    &= \calN(y \,|\, f^{(2)}(f^{(1)}([x_n, z_n])), \sigma^2 I_{P}),
\end{align*}
the model then has joint distribution
\begin{align*}
    &p(y_{n} , f^{(1)}, f^{(2)}, z_{n} ; x_{n}) ~~~~~~~~~~~~~~~~~~~~~
    \\
    &~~=p(y_{n} \,|\, f^{(1)}, f^{(2)}, z_{n} ; x_{n}) \, p(f^{(2)} | f^{(1)}) \, p(f^{(1)})  p(z_{n}),
    \\
    &\text{where} ~~~~~~~
    f^{(1)} \sim \calGP(f^{(1)} \,|\, m_1, K_1),
    \\
    &~~~~~~~~~~~~~~~~\, f^{(2)} \sim \calGP(f^{(2)} \,|\, m_2, K_2),
    \\
    &~~~~~~~~~~~~~~~~~~\,\, z_{n} \sim \calN( 0, I_{\tilde{D}}). 
\end{align*}
Here, $m_1(\cdot)$ and $m_2(\cdot)$ are mean functions, and $K_1(\cdot, \cdot)$ and $K_2(\cdot, \cdot)$ covariance functions.
For simplicity of notation, we will drop any subscripts from identity matrices in the remainder of the paper.

\subsection{Importance-Weighted Variational Inference for Latent-Variable Deep GPs}
\label{sec:iwvi-dgp}

Posterior inference in \dgp models is generally intractable, which necessitates the use of approximate inference methods.
For the \dgp model presented above,~\citet{salimbeni2019iwvi} proposed an \textit{importance-weighted variational inference} (\iwvi) approach, inspired by~\citet{burda2015iwae} (see Section~\ref{sec:iwvae_snr}), to obtain a tractable evidence lower bound (\elbo) that can be used to learn a variational approximation of the posterior.
Specifically, they present a \textit{partially-collapsed} \elbo that provides both lower-variance estimates and a tighter bound as the number of importance samples $K$ increases, namely,
\begin{align}
    \calL_{K}
    \defines&\,
    \E \left[  \sum_{n}\log \frac{1}{K} \sum_{k=1}^K \frac{\calF(x_{n}, y_{n}, f^{(1)}_k, z_{n, k}) p(z_{n, k})}{ q_{\phi}(z_{n, k})} \right]
    \nonumber
    \\
    &\quad - \sum_{\ell=1}^{2} \DKL*{q(f^{(\ell)})}{p(f^{(\ell)})},
    \label{eq:iwvi_elbo_collapsed}
    \\
    \text{where} \nonumber
    \\
    \calF(&x_{n}, y_{n}, f^{(1)}_k, z_{n, k}) 
    \nonumber
    \\
    &\defines
    \exp\left( \E _{q(f^{(2)})}\left[ \log p(y_{n} \,|\, f^{(2)}, f^{(1)}_k, z_{n, k} ) \right]\right),
    \nonumber
\end{align}
which can be calculated analytically for a Gaussian likelihood;  the expectation is over \mbox{$(f^{(1)}_{1:K}, z_{n, 1:K})$}, with \mbox{$z_{n, k} \sim q_{\phi}(z_{n})$} and \mbox{$p(z_{n})=\mathcal{N}( 0, I)$}; and \mbox{$f^{(1)}_k$} are samples from the variational distribution $q(f^{(1)}_{k})$.
See~\Cref{appsec:partially-collapsed_bound_iwvi-dgp} for further details.

\subsection{Importance Weighted Autoencoders and Their Signal-to-Noise Ratio Issues}
\label{sec:iwvae_snr}

\citet{burda2015iwae} propose importance weighted autoencoders (\iwaes) as a means of providing a tighter variational lower bound for \vae training given by
\begin{gather}
    \calL_{\text{IWAE}}
    \defines
    \sum\nolimits_{n=1}^N \E \left[\log \frac{1}{K} \sum\nolimits_{k=1}^K w_{n, k} \right],
    \label{eq:iwvi_vae_bound}
    \\
    \text{where} \quad w_{n, k} \defines \frac{p_\theta(x_{n}, z_{n, k})}{q_\phi(z_{n, k} \,|\, x_{n})},
    \nonumber
\end{gather}
$K$ is the number of importance samples, $N$ is the number of data points, $p_\theta(x_{n}, z_{n, k})$ represents the generative model, $q_\phi(z_{n, k} \,|\, x_{n})$ the amortized inference network, and the expectation is over $\prod_{k=1}^K q_\phi(z_{n, k} \,|\, x_{n})$. 
Training is done by optimizing $\calL_{\text{IWAE}}$ with respect to both $\theta$ and $\phi$ using stochastic gradient ascent.
Assuming that reparameterization of $z_{n, k}$ in $\calL_{\text{IWAE}}$ is possible~\citep{KingmaVAE}, and that it is possible to take mini-batches of the data, the gradient estimates for data point $x_n$ are given by
\begin{gather}
    \Delta_{n, M, K}^\text{IWAE}(\theta,\phi) \defines \frac{1}{M} \sum_{m=1}^{M} \nabla_{\theta,\phi} \log \frac{1}{K} \sum_{k=1}^{K} w_{n, m, k},
    \nonumber %
    \\
    \text{where} \quad z_{n, m, k} \stackrel{\text{i.i.d.}}{\sim} q_{\phi}(z_{n} \,|\, x_{n})
    \nonumber
\end{gather}
and $M$ is the number of samples used for estimation of the outer expectation.
\citet{burda2015iwae} show that increasing $K$ provably tightens the variational lower bound $\calL_{\text{IWAE}}$, and~\citet{rainforth2018nesting} confirm that it also reduces the variance of the resulting estimates.

However,~\citet{rainforth2018bounds} show that the \emph{relative} variance of the gradient estimates with respect to $\phi$ actually \emph{increases} with the number of importance samples $K$:
As $K$ increases, the expected gradient, $\E [\Delta_{n, M, K}^\text{IWAE}(\phi)]$, tends to zero faster than its standard deviation decreases.
To formalize this, they introduce the notion of a signal-to-noise ratio (\snr)
\begin{align*}
    \operatorname{SNR}_{n, M, K}^\text{IWAE}(\psi)=\frac{\left|\E\left[\Delta_{n, M, K}^\text{IWAE}(\psi)\right]\right|}{\sqrt{\operatorname{Var}\left[\Delta_{n, M, K}^\text{IWAE}(\psi)\right]}},
\end{align*}
where $\psi \in \{\phi,\theta\}$, and show that
\begin{align*}
  \operatorname{SNR}_{n, M, K}^\text{{IWAE}}(\theta)
  &=
  \calO\left({\sqrt{MK}}\right)
  \\
  \operatorname{SNR}_{n, M, K}^\text{{IWAE}}(\phi)
  &=
  \calO\left({\sqrt{M/K}}\right),
\end{align*}
such that increasing $K$ decreases the \snr of the inference network, thereby potentially removing its ability to train effectively.
As well as being problematic in its own right, this pathology can further have a knock-on effect on the training of the generative network.

\section{Signal-to-Noise Issues for Deep GPs}
\label{sec:snr_issue}

Given the similarity between the variational bounds in \iwaes and \iwvi for \dgps, it is natural to ask whether the \snr issues in \iwaes also occur in \iwvi for \dgps.
We note here that for \iwaes, this pathology affected the encoder (i.e.,~$\phi$) but not the decoder (i.e.,~$\theta$) and so it is not immediately obvious how these results will translate; \snr issues are not a universal for all gradients of importance-weighted variational bounds.

Inspecting~\Cref{eq:iwvi_elbo_collapsed}, we see that the \dgp setting shares a number of similarities, but also features some key differences to the \iwae setting.
On one hand, both $\calL_{K}$ (\Cref{eq:iwvi_elbo_collapsed}) and $\calL_{\text{IWAE}}$ (\Cref{eq:iwvi_vae_bound}) contain expectations over a logarithm of an importance sampling estimate and both make use of a variational approximation.  
On the other hand, $\calL_{K}$ contains an additional KL term, preventing it from ever achieving a tight bound even when the variational approximation over the latent variable is optimal~\citep{salimbeni2019iwvi}.
Furthermore, it contains the function $\calF(x_{n}, y_{n}, f^{(1)}_{k}, z_{n, k})$ instead of a plain likelihood function, and the outer expectation contains additional stochasticity from the \dgp layers.
More specifically, the posterior predictive distribution \emph{of each layer} needs to be inferred and the approximations produced are functions of the layer's hyperparameters, its variational parameters, and the previous layer's outputs.
Establishing the implications of these differences will be key to unearthing whether \dgps suffer from \snr issues.

To assess if \snr deterioration of the gradient estimates occurs for $\calL_K$ as $K$ increases, we must first identify which gradients are potentially problematic.
Considering~\Cref{eq:iwvi_elbo_collapsed}, it is relatively straightforward to see that, analogously to the decoder in a \vae, the gradients for the \dgp hyperparameters will remain non-zero even as \mbox{$K\to\infty$} and thus will not be susceptible to \snr issues. 
Similarly, it is reasonably straightforward to show that the parameters of $q(f^{(\ell)})$ will not be problematic. 
However, there is a potential for problems to occur for the parameters of the variational distribution over the latent variables, $q_{\phi}(z)$, with variational parameters $\phi$.

It turns out that such a problem does indeed occur for the gradient estimates w.r.t.~$\phi$.
To demonstrate this, the first key step is to reparameterize \emph{all} of the stochastic quantities in the variational bound in~\Cref{eq:iwvi_elbo_collapsed}. 
We can then express the $M$-sample Monte Carlo estimate for the gradient associated with a single data point $x_{n}$ as
\begin{align}
\label{eq:iwvi_dgp_reparameterization_gradient_estimator}
    \Delta_{n, M, K}^{\text{{DGP}}}(\phi) \defines&\, \frac{1}{M} \sum_{m=1}^M \gradphi \log\frac{1}{K} \sum_{k=1}^K w_{n, m, k}, \\
    \text{where} ~~  w_{n, m, k}
    &\defines
    \frac{\calF(x_{n}, y_{n}, f^{(1)}_{m,k}, z_{n, m, k}) p(z_{n, m, k})}{ q_{\phi}(z_{n, m, k})}, \nonumber
    \\
    z_{n, m, k} &\sim q_{\phi}(z_{n}), \quad f^{(1)}_{m,k} \sim q(f^{(1)}). \nonumber
\end{align}
We are now ready to present our main result which shows that \iwvi for \dgps suffers from an \snr issue analogous to that seen in \iwaes.
\begin{theorem}[Asymptotic \snr in \iwvi for \dgps]\label{thm:snr_iwvi-dgp}
Let $w_{n, m, k}$ be as defined in~\Cref{eq:iwvi_dgp_reparameterization_gradient_estimator}.
Assume that when $M=K=1$, the expectation and variance of the gradients estimates in~\Cref{eq:iwvi_dgp_reparameterization_gradient_estimator} are non-zero, and that the first four moments of $w_{n, 1,1}$ and $ \gradphi w_{n, 1,1}$ are all finite and that their variances are also non-zero.
Then the signal-to-noise ratio of each $\Delta_{n, M, K}^{\text{{DGP}}}(\phi)$ converges at the following rate
\begin{align}
\begin{split}
\label{eq:full_bound}
    &\hspace{-4pt}\operatorname{SNR}_{n,M, K}^{\textrm{\emph{DGP}}}(\phi) \\
    &=\sqrt{M}\left|\frac{\gradphi \operatorname{Var}\left[w_{n, 1,1}\right]+\calO\left(\frac{1}{K}\right)}{2 Z_n \sqrt{K} \sqrt{\operatorname{Var}\left[\gradphi w_{n, 1,1}\right]}+\calO\left(\frac{1}{\sqrt{K}}\right)}\right|,
\end{split} \\
&= \calO\left({\sqrt{M/K}}\right), \nonumber
\end{align}
where $Z_n \defines \E [w_{n,1,1}]$ is a lower bound on the marginal likelihood of the $n^{\textnormal{th}}$ data point.
\end{theorem}

\begin{proof}
We start our proof by noting that the average of the importance weights $w_{n,m,k}$ for a given $n$ and $m$,
\begin{align}\label{eq:Zhat_iwvi-dgp}
	\Zhat_{n,m,K} &\defines \frac{1}{K} \sum\nolimits_{k=1}^K w_{n,m,k} ,
\end{align}
is an unbiased Monte Carlo estimator of $Z_n$ as follows
\begin{align}
\begin{split}\label{eq:iwvi_dgp_Z}
    &\E [\hat{Z}_{n,m,K}]
    = \E [w_{n,m,1}]
    = \E [w_{n,1,1}]
    = Z_n.
\end{split}
\end{align}
Moreover, we also have $\lim_{K \rightarrow \infty} \Zhat_{n,m,K} = Z_n, ~\forall n,m$.

\begin{figure*}[t!]
    \centering
    \begin{subfigure}{0.49\textwidth}
    \includegraphics[trim={50pt 10pt 50pt 10pt},clip,width=0.95\linewidth, keepaspectratio]{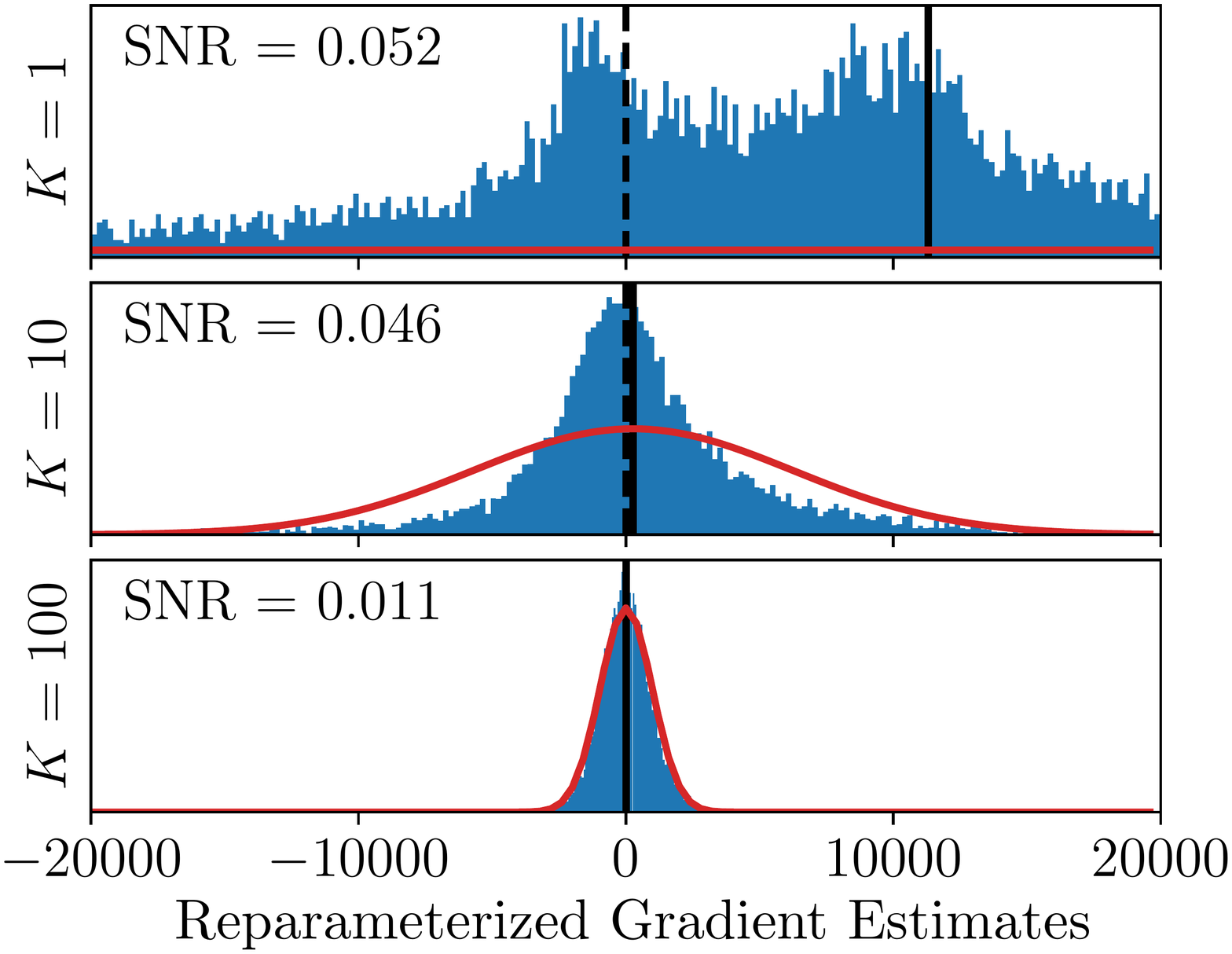}
    \caption{
    Histogram of \reg estimates.
    }
    \label{fig:gradient-histograms-dreg-off}
    \end{subfigure}\hfill
    \begin{subfigure}{0.49\textwidth}
    \centering
    \includegraphics[trim={50pt 10pt 50pt 10pt},clip,width=0.95\linewidth, keepaspectratio]{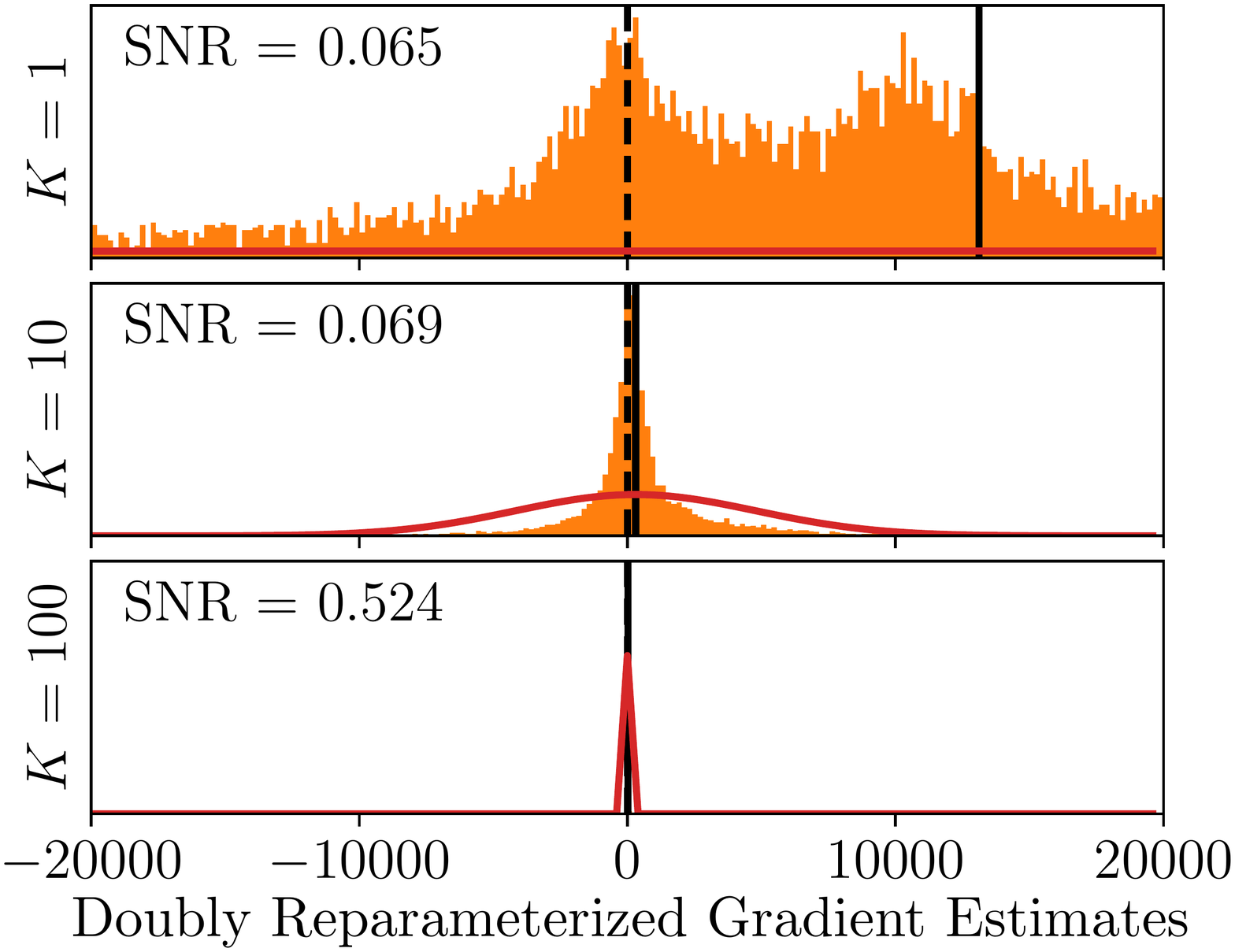}
    \caption{
    Histogram of \dreg estimates.
    }
    \label{fig:gradient-histograms-dreg-on}
    \end{subfigure}
    \caption{
    Histograms of \reg (a) and \dreg (b) estimates for varying numbers of importance samples, $K$, with \mbox{$K=1,10,100$}.
    The solid vertical lines denote the mean of the empirical distributions of gradient estimates, and the dashed vertical lines mark zero. The red curves denote Gaussian probability densities with the same mean and standard deviation as the empirical distributions.
    Note that we have truncated the histograms for $K=1$ in order to use a consistent $x$-axis for each row, see
   ~\Cref{app-fig:gradient-histograms-sep-axes}
    (in~\Cref{appsec:additional_results})
    of a copy of this figure with a separate $x$-axis for each value of $K$.
    }
    \label{fig:gradient-histograms}
    \vspace*{-8pt}
\end{figure*}

The key to establishing that the \snr pathology occurs in \iwvi for \dgps, is now to show that $Z_n$ is independent of $\phi$, the parameters of the variational distribution over the latent variable $z$.
For \iwaes this was trivially true by construction: The expected importance weight is just the marginal likelihood of the model.
In the \dgp case, however, the dependence of $Z_n$ on $\phi$ is not immediately obvious since the dependence of $\calF(x_{n}, y_{n}, f^{(1)}, z_{})$ on $z_{}$ is more complicated than in the \iwae case.
In particular, in \iwvi for \dgps, $z$ acts as an input to the approximate posterior predictive distribution of the \dgp layer $f^{(1)}$, which itself needs to be inferred and with respect to which the expectation over the importance weight in~\Cref{eq:iwvi_dgp_Z} is taken.

In short, we need to demonstrate that the following condition holds despite the dependence of $q(f^{(1)})$ on $z$:
\begin{customcondition}{1}\label{condition:gradZzero}
$\gradphi Z_n= 0~ \forall n$.
\end{customcondition}
\vspace{-5pt}
To examine if this condition is indeed satisfied in \iwvi for \dgps, the first hurdle that we need to overcome is that $f^{(1)}$ and $f^{(2)}$ represent stochastic functions, which means it is difficult to reason about their gradients or concretely establish the dependency relationships between variables.
To get around this, the key step is to realize that we can reparameterize this stochasticity and then view the bound from the perspective of the pushforward distribution induced by passing an input--latent pair through the resulting realizations of the functions.
Specifically, we have, for an arbitrary $z$,
\begin{align}
\begin{split}
    f^{(1)}&(x_n,z)
    =
    \ftilde(\epsilon, x_n, z, \psi^{(1)})
    \\
    \defines\,
    &\mu_{}^{f^{(1)}}(x_n, z, \psi^{(1)}) + \epsilon \odot \sqrt{\Sigma_{}^{f^{(1)}}(x_n, z, \psi^{(1)})},
\end{split}
\end{align}
where $\ftilde(\cdot)$ is a deterministic function;
$\epsilon \sim \calN(0, 1)$; $\mu_{}^{f^{(1)}}$, $\Sigma_{}^{f^{(1)}}$ denote the mean and covariance functions of the predictive distribution over $f^{(1)}$, respectively; and $\psi^{(1)}$ denotes the variational parameters associated with $f^{(1)}$.
We thus have
\begin{align*}
    Z_n &
    =\E_{q_{\phi}(z) p(\epsilon)} \left[ \frac{\calF(x_{n}, y_{n}, \ftilde(\epsilon, x_n, z, \psi^{(1)})) p(z)}{ q_{\phi}(z)} \right]
    \\
    &=
    \E_{p(z) p(\epsilon)} \left[ \calF(x_{n}, y_{n}, \ftilde(\epsilon, x_n, z, \psi^{(1)})) \right].
\end{align*}
From here it is now straightforward to see that \mbox{$\gradphi Z_n = 0$} as the final form of $Z_n$ above has no direct or indirect dependency on $\phi$.
We thus see that~\Cref{condition:gradZzero} is satisfied.

This now allows us to invoke the proof of~\citet[Theorem 1]{rainforth2018bounds} because a) our gradient estimator in~\Cref{eq:iwvi_dgp_reparameterization_gradient_estimator} is equivalent to that of \iwae except in the distribution of $w_{n,m,k}$, and b) the only VAE-context specific part of their proof is in showing that~\Cref{condition:gradZzero} holds.

Starting from equations (29) and (31) in their proof and using $\gradphi Z = 0$, we can now derive the mean and variance of the gradient estimator of a single data point (i.e.~for a particular $n$) as
\begin{align*}
    &\E [\Delta_{n, M, K}^{\text{{DGP}}}(\phi)] \!=\! -\frac{1}{2K Z_n^2} \gradphi \left(\operatorname{Var}[w_{n,1,1}]\right) \!+\! \mathcal{O}(K^{-2})\\
    &\operatorname{Var}[\Delta_{n, M, K}^{\text{{DGP}}}(\phi)] \!=\! \frac{1}{MKZ_n^2} \E [\left(\nabla_{\phi} w_{n,1,1}\right)^2]\!+\!\frac{\mathcal{O}(K^{-2})}{M}.
\end{align*}
Because $w_{n,1,1}$ has been reparameterized, we also have $\E [\left(\nabla_{\phi} w_{n,1,1}\right)^2]=\operatorname{Var}\left[\gradphi w_{n, 1,1}\right]$.
Substituting these into the definition of the \snr yields
\begin{align*}
\operatorname{SNR}_{n,M, K}^{\textrm{DGP}}(\phi) = 
\left|
\frac{-\frac{1}{2K Z_n^2} \gradphi \left(\operatorname{Var}[w_{n,1,1}]\right) \!+\! \mathcal{O}(K^{-2})}
{\sqrt{\frac{1}{MKZ_n^2} \operatorname{Var} [\nabla_{\phi} w_{n,1,1}]\!+\!\frac{\mathcal{O}(K^{-2})}{M}}}
\right|
\end{align*}
from which~\Cref{eq:full_bound} follows by straightforward manipulations.
\end{proof}
\begin{remark}
Note that the additional results of~\citet[Section 3.2]{rainforth2018bounds}, which show that the \snr improves as more data points are used in the gradient calculation for~\iwaes (namely as $\mathcal{O}(\sqrt{N})$), do not directly carry over to the \dgp setting because $f^{(1)}$ induces correlations between the different $\Delta_{n, M,K}^{\text{{DGP}}}(\phi)$.
\end{remark}

\begin{figure*}[t!]
    \vspace{9pt}
    \centering
    \includegraphics[width=\linewidth, keepaspectratio]{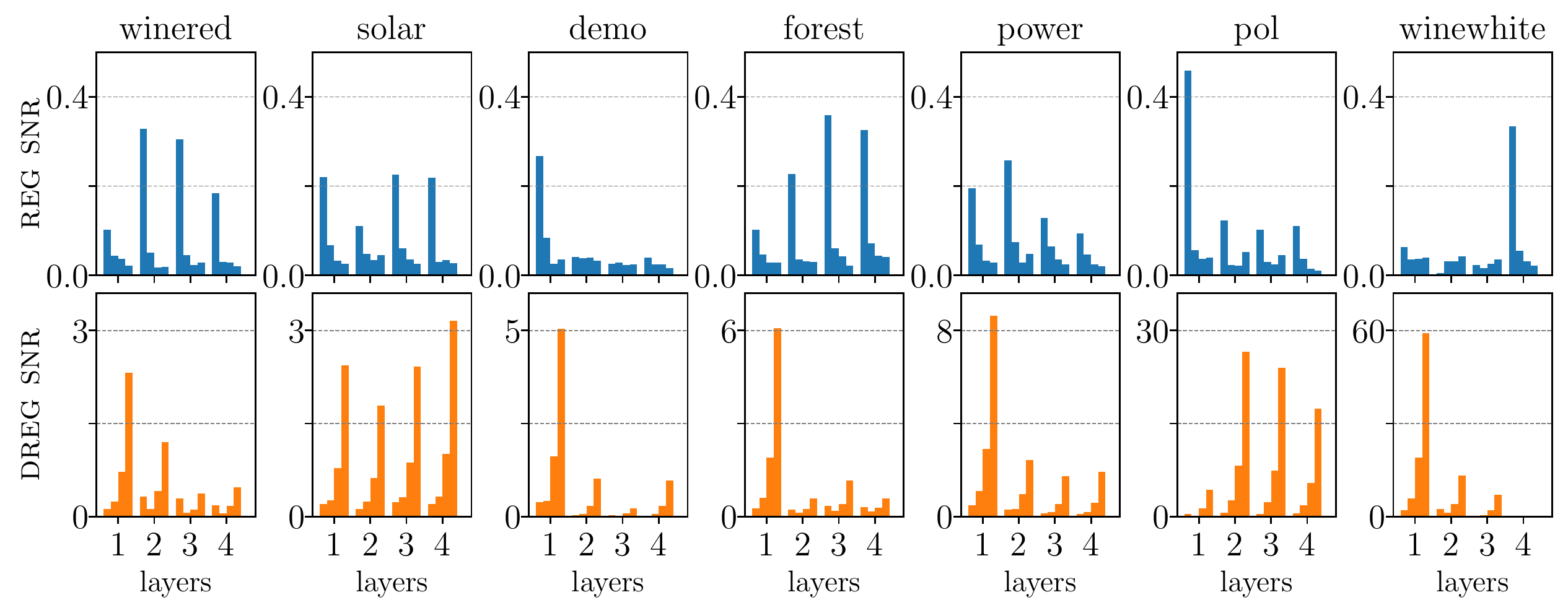}
    \caption{
        \snr of reparameterization (top row) and doubly reparameterized (bottom row) gradient estimates for shallow \gps and \dgps of 2-4 layers on a selection of real-world datasets.
        The labels on the $x$-axes correspond to the depths of the models.
        The bars for each depth show the \snr for increasing numbers of importance samples, $K=1,10,100,1000$, from left to right.
        In the top row, for \textsc{(d)gp}s of any depth, larger $K$ tends to correspond to lower \snrs.
        In the bottom row, for \textsc{(d)gp}s of any depth, larger $K$ tends to correspond to higher \snrs.
        Note the difference in $y$-axis scales across plots in the bottom row.
        See~\Cref{appsec:experimental_details} for further experiment details.
    }
    \label{fig:snr-grid}
\end{figure*}

\subsection{Empirical Confirmation}
\label{sec:empirical_confirmation_snr}

We now investigate the signal-to-noise ratio of the reparameterization gradient (\reg) estimates in \iwvi for \dgps empirically, and show that the empirical results are consistent with~\Cref{thm:snr_iwvi-dgp}.

We begin by examining the behavior of the gradient estimates for the parameters $\phi$ as we increase the number of importance samples, $K$, when training a two-layer latent-variable \dgp using \iwvi.
For illustrative purposes, we start by investigating the convergence behavior of the gradient estimates on an easy-to-visualize synthetic dataset specifically designed to exhibit non-Gaussianity in the target values.
We train the \dgp until the variational parameters are near convergence and then take $10,000$ gradient samples for each parameter in $\phi$ using the reparameterization gradient estimator.
For further details about the experiment setup and a visualization of the dataset, see~\Cref{appsec:experimental_details}.
In~\Cref{fig:gradient-histograms-dreg-off}, we present histograms of the empirical distribution of gradient estimates for a single (representative) variational parameter of $q_{\phi}$.

As we would expect by the Central Limit Theorem and knowledge that~\Cref{condition:gradZzero} is satisfied, the mean and the standard deviation of the distributions of gradient estimates approach zero as $K$ increases.
However, the means of the gradient estimates appear to be approaching zero more rapidly than their standard deviations, which would suggest a decrease in the \snr as $K$ increases.
To assess if such a deterioration in fact occurs, we compute the mean \snr across parameters $\phi$ for varying $K$ and multiple \dgp depths on a range of real-world datasets.
We collect the resulting mean \snrs in the top row of~\Cref{fig:snr-grid}, and find that---across datasets and \dgp depths---the \snr consistently approaches zero as $K$ increases, confirming that \iwvi for \dgps \emph{does} indeed suffer from \snr deterioration.

\section{Quantifying \& Fixing the SNR Pathology}
\label{sec:snr_analysis}

In the previous section, we demonstrated that \iwvi for \dgps suffers from \snr deterioration as the number of importance samples is increased.
This result leads to two questions:
(i) Can we construct an alternative estimator that avoids these issues; and (ii) will fixing the \snr issue lead to improved predictive performance?

To address the former question, we adapt the double-reparameterized \iwae estimator of~\citet{tucker2019dreg} to \dgp models and show that the resulting \dreg estimator completely remedies the \snr issue.
We then use this estimator to asses the impact \snr issues have had on the quality of the \dgp's posterior predictive distribution.
We consider a selection of datasets used in \citet{salimbeni2019iwvi} and find that our fix often leads to a statistically significant \emph{improvement in predictive performance} for \iwvi in \dgps.
Specifically, we find that on tasks where fitting the data well requires a predictive distribution with \emph{non-Gaussian marginals}, such that the latent variables are necessary, fixing the \snr deterioration improves the predictive performance.

\begin{figure*}[t!]
    \centering
    \includegraphics[width=\linewidth, keepaspectratio]{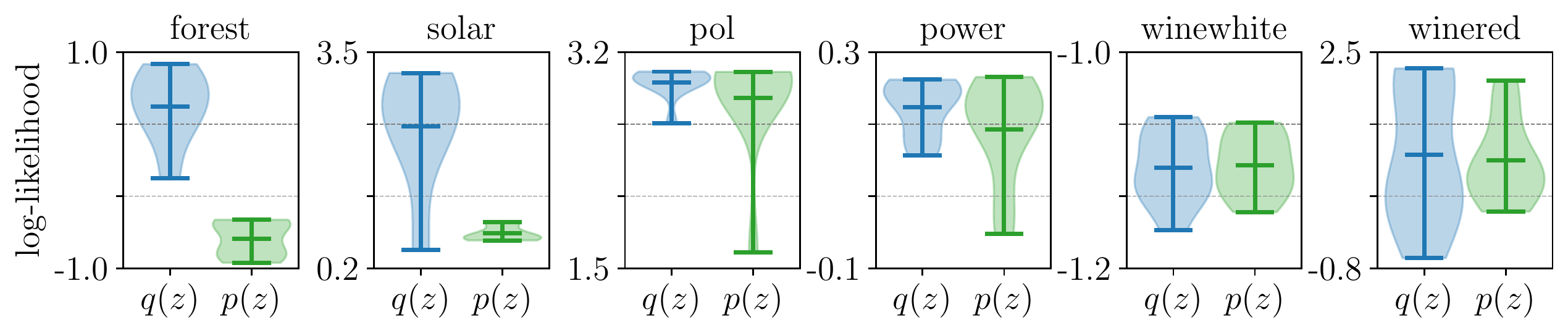}
    \caption{
        Comparison of predictive performance of $2$-layer \dgps with a learned variational distribution over the latent variable (left of each pair, \textcolor{reg}{blue}) and a variational distribution over the latent variable fixed to the prior (right of each pair, \textcolor{highlight}{green}).
        The shaded area shows the range of test log-likelihoods over 10 train--test splits, with the width indicating the distribution over the range.
        The central horizontal lines in each plot show the mean.
    }
    \label{fig:learned_vs_fixed_q}
    \vspace{-10pt}
\end{figure*}

\subsection{Avoiding SNR Deterioration in Deep GPs}
\label{sec:snr_avoidable}

To avoid \snr deterioration, we extend prior work to derive a doubly reparameterized gradient (\dreg,~\citep{tucker2019dreg}) estimator for \iwvi in \dgps.
This gradient estimator is equal to the \reg estimator (i.e.~\Cref{eq:iwvi_dgp_reparameterization_gradient_estimator}) in expectation, but does not suffer from asymptotic deterioration of the \snr as $K$ increases.
Assuming that reparameterization of $q(z_{n})$ in $\calL_{K}$ is possible, the \dreg estimator of $\calL_{K}$ at a single data point $x_{n}$ can be expressed as (see~\Cref{appsec:dreg_derivation} for the derivation)
\begin{align}\label{eq:iwvi_dgp_dreg}
\begin{split}
    &\widetilde{\Delta}_{n, M, K}^{{\text{DGP}}}(\phi)
    \\
    \defines &\frac{1}{M} \sum_{m=1}^M  \sum_{k=1}^K \left( \frac{{w}_{n,m, k}}{\sum_{j=1}^K {w}_{n,m, j}} \right)^2
    \frac{\partial \log {w}_{n,m, k}}{\partial z_{n,m,k}}
    \frac{\partial z_{n,m, k}}{\partial \phi},
\end{split}
\end{align}
where $w_{n,m,k}$ is as before.
As we explain in~\Cref{appsec:dreg_derivation}, analogously to the results of~\citet{tucker2019dreg}, the \snr of this gradient estimator scales as $\calO(\sqrt{K})$ instead of $\calO(1/\sqrt{K})$, that is, the \snr \emph{improves} as $K$ increases.
The \dreg estimator can be used as a drop-in replacement of the \reg estimator and is guaranteed to be at least as good or better without incurring additional computational cost \citep{tucker2019dreg}.\footnote{Our implementation of the \dreg estimator follows \url{https://sites.google.com/view/dregs}.}

To show that the \dreg estimates of the latent-variable variational parameters do not suffer from \snr deterioration in practice, we revisit the empirical investigations carried out in~\Cref{sec:empirical_confirmation_snr} and compute the \snr of the gradient estimates for an increasing number of importance samples across datasets and \dgp depths.
\Cref{fig:snr-grid} shows the resulting \snrs.
We find that the effect of increasing $K$ for \reg and \dreg is markedly different:
Unlike for the \reg estimator, the \dreg estimate \snr values \emph{increase} with $K$.

The difference between the \snrs of the two gradient estimators is explained by the speed at which the mean and standard deviation of the empirical distributions over the gradient estimates converge to zero as $K$ increases.
\Cref{fig:gradient-histograms-dreg-on} shows the empirical distribution of \dreg estimates.
As can be seen in the histograms, the means for $K=1,10,100$ decrease as $K$ increases, but the standard deviations of the empirical distributions of \dreg estimates are significantly smaller than those of the empirical distributions of the \reg estimates shown in~\Cref{fig:gradient-histograms-dreg-off}.
This difference is particularly striking for $K=100$, where the \dreg estimates are so peaked that they are difficult to visualize without changing the $x$-axis range.
The upshot of the results in Figures~\ref{fig:gradient-histograms} and~\ref{fig:snr-grid} is that the \dreg estimator completely remedies the \snr issue exhibited by the \reg estimates.

\subsection{Does the SNR Issue Affect Training and Predictive Performance?}
\label{sec:snr_effect_on_performance}

Armed with the \dreg estimator, we are now able to investigate the impact of correcting the \snr deterioration on the training and test performance of the model.
To do so, we consider a selection of datasets for which we either expect or do not expect the \snr issue to lead to a deterioration in predictive performance.
In particular, we assess the impact of \snr deterioration on two datasets where fitting the data well requires a predictive distribution with \emph{non-Gaussian marginals} (`forest' and `solar'), two datasets which can be fit well with just Gaussian marginals (`winewhite' and `winered'), and two datasets where a predictive distribution with non-Gaussian marginals could potentially help a modest amount (`pol' and `power').
By construction, latent-variable \dgps are able to learn highly non-Gaussian marginals and so learning a variational distribution over the latent variable should positively affect predictive performance whenever highly non-Gaussian marginals lead to a better fit of the data.
Hence, we would expect the effect of a deterioration in the \snr of the gradient estimates to be highest whenever this is the case.

\begin{figure}[t!]
    \centering
    \includegraphics[width=\linewidth, keepaspectratio]{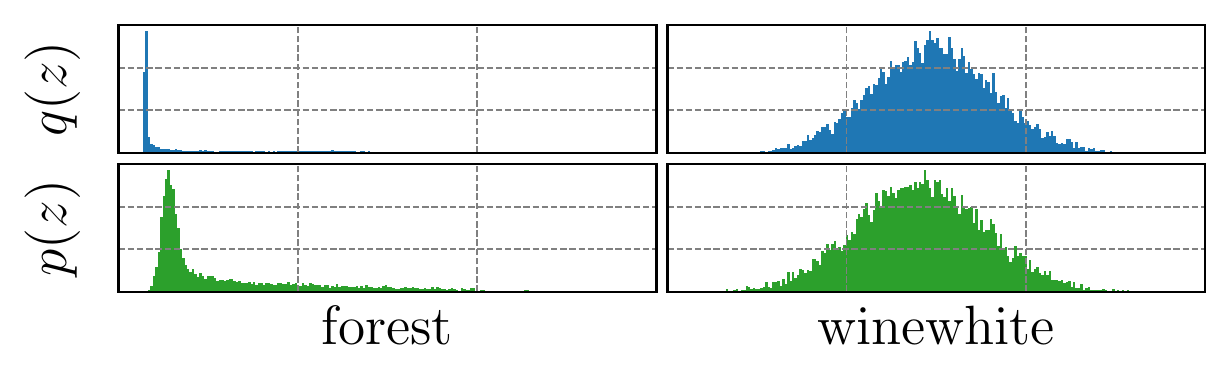}
    \vspace{-15pt}
    \caption{Marginal predictive distributions of $2$-layer \dgps with a learned variational distribution over the latent variable (top row, \textcolor{reg}{blue}) and a variational distribution over the latent variable fixed to the prior (bottom row, \textcolor{highlight}{green}) for randomly selected test points from the `forest' and `winewhite' datasets.
    We note that learning the variational distribution helps fitting a better non-Gaussian marginal distribution for `forest' (as evidenced by the test log-likelihoods shown in~\Cref{fig:learned_vs_fixed_q}), whereas for `winewhite' the marginal is Gaussian and fit equally well in both cases.
    }
    \label{fig:reg_vs_fixq_marginals}
    \vspace{-15pt}
\end{figure}

In~\Cref{fig:learned_vs_fixed_q}, we present the plots of test log-likelihoods for latent-variable \dgps with learned variational distributions over the latent variable (left of each pair, in \textcolor{reg}{blue}) and with latent variable distributions fixed to standard Gaussian priors (right of each pair, in \textcolor{highlight}{green}).
As can be seen from the plots, the importance of learning a variational distribution for the latents sometimes has a significant effect on predictive performance, but sometimes has no noticeable effect at all.
In~\Cref{fig:reg_vs_fixq_marginals}, we present sample marginal distributions of the best-performing models, which indicate that learning a distribution over the latents, $q_{\phi}(z)$, has a larger effect the more non-Gaussian the predictive distribution's marginals need to be.
We thus expect our \snr fix to be helpful for `forest' and `solar', but not for `winewhite' and `winered.'

\setlength{\tabcolsep}{9.3pt}
\begin{table*}[t!]
\small
    \centering
    \caption{
    Comparison of predictive performance of two-layer \dgps trained with \reg and \dreg estimators.
    We choose $K = 50$ to ensure the \snr deterioration occurs, which~\Cref{fig:snr-grid} shows is the case for $K \geq 10$.
    For each dataset, we provide the mean \elbos on the training dataset and log-likelihoods on the test dataset over $20$ random train--test splits as well as the corresponding standard errors.
    Boldface indicates higher means.
    The rightmost column shows $p$-values for one-sided Wilcoxon signed-rank hypothesis tests on the log-likelihoods, as described in~\Cref{sec:snr_effect_on_performance}.
    }
    \label{tab:reg_dreg_performance}
    \vspace{3pt}
    \begin{tabular}{r r c r c r c r c r}
        \midrule
        & \multicolumn{4}{c}{Train $\elbo$ ($K=50$)} & \multicolumn{5}{c}{Test log-likelihood} \\
        Dataset & \multicolumn{2}{c}{\reg} & \multicolumn{2}{c}{\dreg} & \multicolumn{2}{c}{\reg} & \multicolumn{2}{c}{\dreg} & Wilcoxon Test \\ \midrule
        & \multicolumn{1}{c}{Mean} & \multicolumn{1}{c}{SE} &
        \multicolumn{1}{c}{Mean} & \multicolumn{1}{c}{SE} &
        \multicolumn{1}{c}{Mean} & \multicolumn{1}{c}{SE} &
        \multicolumn{1}{c}{Mean} & \multicolumn{1}{c}{SE} &
        $p$-value \\
        \midrule
        forest & -97.56 & (11.04) & \textbf{-92.53} & (10.42) & 0.59 & (0.08) & \textbf{0.63} & (0.08) & 0.1\% \\
        solar & 1657.41 & (27.56) & \textbf{1707.75} & (42.20) & 2.33 & (0.17) & \textbf{2.57} & (0.11) & 2.8\% \\
        pol & 34610.49 & (66.18) & \textbf{34665.08} & (70.34) & 2.99 & (0.01) & 2.99 & (0.01) & 24.7\% \\
        power & 1510.50 & (10.62) & \textbf{1515.60} & (10.16) & 0.21 & (0.01) & 0.21 & (0.01) & 67.3\% \\
        winewhite & \textbf{-4701.26} & (4.92) & -4703.14 & (4.98) & -1.11 & (0.01) & -1.11 & (0.01) & 50.0\% \\
        winered & \textbf{447.91} & (249.81) & 314.75 & (216.32) & 0.57 & (0.27) & \textbf{0.61} & (0.20) & 41.1\% \\
        \midrule
        \multicolumn{7}{r}{} & \multicolumn{2}{l}{Across Datasets:} & 1.2\% \\
        \midrule
        \midrule
    \end{tabular}
    \vspace*{-10pt}
\end{table*}

To test this, we train latent-variable \dgps with both \reg and \dreg estimators on the six datasets shown in~\Cref{fig:learned_vs_fixed_q} and compare the resulting train $\elbo$s and test log-likelihoods.
For these experiments, we consider a two-layer \dgp with the hyperparameters that directly affect the \snr---the number of importance samples $K$ and the minibatch size---set to $50$ and $64$, respectively.
\Cref{tab:reg_dreg_performance} shows a summary of the results.
We see that our \dreg estimator provides improvements for the datasets where the latent variational approximation is clearly important (`forest' and `solar'), while performing similarly to the \reg estimator when it is not.

\begin{figure}[b!]
    \centering
    \vspace*{-20pt}
    \includegraphics[width=\linewidth, keepaspectratio]{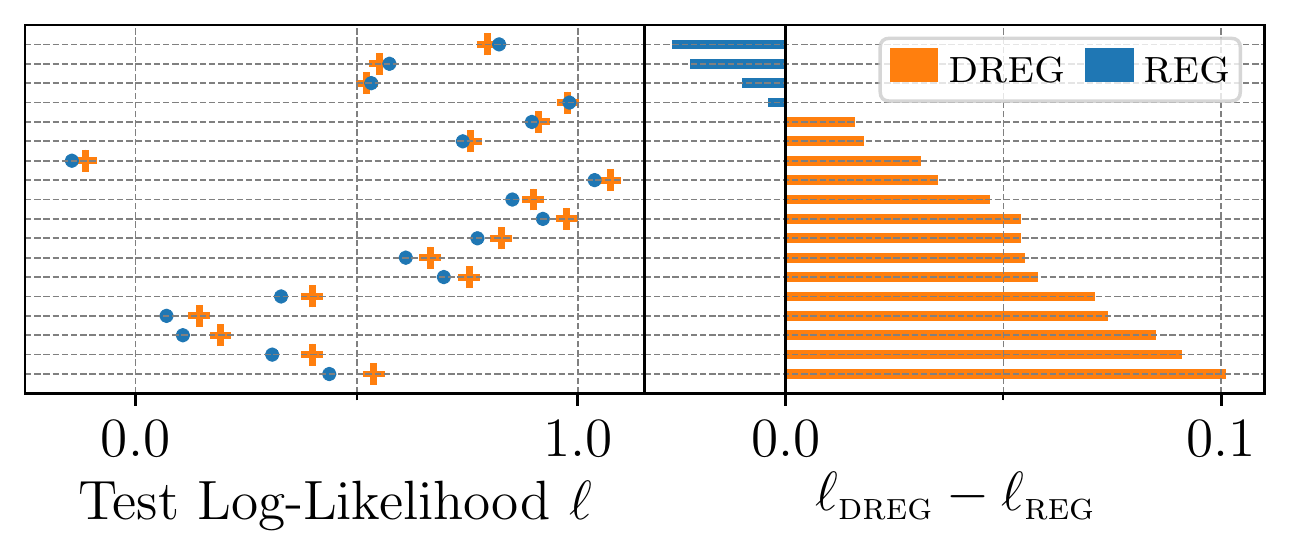}
    \vspace{-20pt}
    \caption{
        Left: Test log-likelihoods across seeds on the `forest' dataset, ordered from most negative to most positive.
        Right: Corresponding absolute difference in test log-likelihood between \reg and \dreg across seeds.
    }
    \label{fig:forest_tll_comparison}
    \vspace{-5pt}
\end{figure}

To assess whether the differences are statistically significant, we perform Wilcoxon signed-rank tests for each dataset individually as well as across datasets~\citep{wilcoxon1992individual}.
This is chosen in preference to a more conventional $t$-test because of the highly non-Gaussian nature of the variations across seeds, such that the criteria for a $t$-test to be representative are not met.
Specifically, for each dataset and across datasets, we test the null hypothesis that the difference in test log-likelihoods under \reg and \dreg for each random seed is zero with the alternative hypothesis that the test log-likelihood under \dreg is greater than the test log-likelihood under \reg.
We present the results from this one-sided hypothesis test in~\Cref{tab:reg_dreg_performance} (rightmost column).
As shown in the table, the $p$-values for the datasets on which learning a variational distribution over the latent variable yields a better fit of the data (see~\Cref{fig:learned_vs_fixed_q}) are $0.1$\% and $2.8$\%, leading us to reject the null hypothesis for both at the standard $5$\% confidence level.
While the statistical significance of the improvement in log-likelihood under \dreg is obfuscated by large standard errors across train-test splits,~\Cref{fig:forest_tll_comparison} shows that \dreg leads to a \emph{consistent improvement} in predictive performance across splits when the variational distribution is important.
We further observe that, as expected, when the variational distribution is not important for obtaining a good fit, there is no statistically significant change.

Finally, we find that the improvement in predictive performance under \dreg \emph{across datasets} for even a moderately small number of importance samples ($K=50$) is statistically significant at the $5$\% confidence level (with a $p$-value of $1.2$\%), which is expected since, as~\Cref{tab:reg_dreg_performance} shows, fixing the \snr issue with the \dreg estimator either improves predictive performance or does not affect it at all.

We finish by noting that that~\citet{salimbeni2019iwvi} already provided comparisons to other inference approaches for \dgps \citep{salimbeni2017dsvi, havasi2018hmc}---along with other regressors---and found that using \iwvi for \dgps with the \reg estimator produced state-of-the-art predictive performance.
It thus follows that our approach is able to further improve upon the state of the art for regression modeling and in especially well-suited for data that requires complex, highly-Gaussian predictive marginals.

\section{Conclusions}
\label{sec:conclusions}

We have shown that the gradient estimates used in training \dgps with \iwvi are susceptible to \textit{signal-to-noise ratio} issues:
We demonstrated theoretically that the \snr for the latent-variable variational parameters increases as more importance samples are used, and confirmed this result empirically.
We have further shown how this pathology can be remedied by adapting the doubly reparameterized approach of~\citet{tucker2019dreg} to the \dgp setting, resulting in a gradient estimator whose \snr \textit{increases} with the number of importance samples for all variables.
This estimator can be used as a drop in replacement without incurring any significant increase in computational cost, implementation challenges, or other negative effects.
We find that it can provide improvements in predictive performance, even for small numbers of importance samples, when variational approximation of the latents is important, without damaging performance when it is not.

\section*{Acknowledgements}
Tim G. J. Rudner is funded by the Rhodes Trust and the Engineering and Physical Sciences Research Council (EPSRC).
Oscar Key is funded by the Engineering and Physical Sciences Research Council (EPSRC), grant number EP/S021566/1.
We gratefully acknowledge donations of computing resources by the Alan Turing Institute.

\bibliography{references}
\bibliographystyle{include/icml2021}

\clearpage

\onecolumn

\begin{appendices}

\crefalias{section}{appsec}
\crefalias{subsection}{appsec}
\crefalias{subsubsection}{appsec}

\setcounter{equation}{0}
\renewcommand{\theequation}{\thesection.\arabic{equation}}

\section*{\LARGE Supplementary Material}
\label{sec:app}

\section{Derivation of the Partially-Collapsed Variational Bound}
\label{appsec:partially-collapsed_bound_iwvi-dgp}

In this section, we present the derivation of the partially-collapsed lower bound presented in \citet{salimbeni2019iwvi}.
As in~\Cref{sec:dgp_model} in the main text, we consider the augmented joint distribution
\begin{align}
    \label{app-eq:dgp_joint}
    p(y , f^{(1)}, f^{(2)}, z) = \prod_n^N p(y_n | f^{(1)}, f^{(2)}, z; x_n) \prod_{\ell = 1}^2 p\left(f^{(\ell)} ; f^{(\ell-1)} \right),
\end{align}
where $\{(x_n, y_n)\}_{n=1}^N$ are input output pairs.
To perform variational inference in such a \dgp model with latent variables, we can maximize an evidence lower bound (\elbo) on the log marginal likelihood
\begin{align}
\label{eq-app:marginal_likelihood}
    p(y) = \int p(y | f^{(2)}, f^{(1)}, z ) \prod_{\ell=1}^2 p\left(f^{(\ell)}\right) p\left(z \right) d z d f^{(1)} d f^{(2)},
\end{align}
by repeatedly applying Jensen's inequality, which yields the variational bound
\begin{align}
    \label{app-eq:iwvi_elbo_standard}
    \log p(y) \geq \sum_n \underset{f^{(1)}, f^{(2)}, z_{n}}{\mathbb{E}} \left[ \log p\Big(y \Big|f^{(2)}, f^{(1)}, z_{n}\Big) \frac{\prod_{\ell}^{2} p(f^{(\ell)}) p\left(z_{n}\right)}{\prod_{\ell}^{2} q(f^{(\ell)}) q\left(z_{n}\right)} \right],
\end{align}
where the expectation is taken over the variational distributions $q(f^{(2)})$ $q(f^{(1)})$, and $q(z_{n})$.
However gradient estimates of this objective may have high variance, since evaluating the log expected likelihood requires Monte Carlo sampling over the final \dgp layer (which cannot be evaluated analytically).

To obtain a lower-variance gradient estimator, \citet{salimbeni2019iwvi} derive an alternative, partially-collapsed lower bound.
To do so, they apply Jensen's inequality to the expectation over $f^{(2)}$ in~\Cref{eq-app:marginal_likelihood} to get
\begin{align}
\label{app-eq:first_jensen}
    \log p(y | f^{(1)}, z) \geq \sum_{n} L_{n}(f^{(1)}, z_{n}) - \DKL{q(f^{(2)})}{p(f^{(2)})},
\end{align}
where
\begin{align}
\SwapAboveDisplaySkip
    L_{n}(f^{(1)}, z_{n}) \defines \E_{f^{(2)}} [ \log p(y_{n} | f^{(2)}, f^{(1)}, z_{n} ) ],
\end{align}
which can be obtained in closed form for a Gaussian likelihood.
Applying the exponential function to both sides of~\Cref{app-eq:first_jensen} then yields
\begin{align}
\SwapAboveDisplaySkip
\label{app-eq:substitute}
    p(y | f^{(1)}, z) \geq \exp \left( \sum_{n} L_{n}(f^{(1)}, z_{n}) - \mathrm{KL}(q(f^{(2)}) || p(f^{(2)})) \right).
\end{align}
Expressing the joint distribution with $f^{(2)}$ integrated out, we can write
\begin{align}
    p(y) = \underset{f^{(1)}, z}{\mathbb{E}} \left[ p(y | f^{(1)}, z ) \frac{p\left(f^{(1)}\right) p\left(z\right)}{q\left(f^{(1)}\right) q\left(z\right)}  \right],
\end{align}
and substituting in~\Cref{app-eq:substitute} and get the lower bound
\begin{align}
    p(y) \geq  \E_{f^{(1)}, z} \exp \left( \sum_{n} L_{n}(f^{(1)}, z_n) - \mathrm{KL}(q(f^{(2)}) || p(f^{(2)})) \right) \frac{p(z) p(f^{(1)})}{q(z) q(f^{(1)})}.
\end{align}
Applying Jensen's inequality for the expectation over $f^{(1)}$ and taking the average over $K$ samples from $z$ to reduce the variance of the resulting bound, \citet{salimbeni2019iwvi} obtain the partially-collapsed lower bound
\begin{align}
\label{app-eq:iwvi_elbo_collapsed}
\begin{split}
    \log p(y)
    &\geq \sum_{n} \underset{f^{(1)}_{1:K}, z_{n, 1:K}}{\mathbb{E}} \left[ \log \frac{1}{K} \sum_{k=1}^K \frac{e^{\mathbb{E}_{f^{(2)}}\left[ \log p(y_n | f^{(2)}, f^{(1)}_{k}, z_{n, k} ) \right]} p\left(z_{n, k} \right)}{ q\left(z_{n, k} \right)} \right] - \sum_\ell^{2} \DKL{q(f^{(\ell)})}{p(f^{(\ell)})}
    \\
    &= \sum_{n} \underset{f^{(1)}_{1:K}, z_{n, 1:K}}{\mathbb{E}} \left[ \log \frac{1}{K} \sum_{k=1}^K \frac{\calF(x_{n}, y_{n}, f^{(1)}_{k}, z_{n, k}) p\left(z_{n, k} \right)}{ q\left(z_{n, k} \right)} \right] - \sum_\ell^{2} \DKL{q(f^{(\ell)})}{p(f^{(\ell)})}
    \\
    &\defines \mathcal{L}_{\text{DGP-IWVI}},
\end{split}
\end{align}
where $\calF(x_{n}, y_{n}, f^{(1)}_{k}, z_{n, k}) \defines \exp\left( \E _{q(f^{(2)})}\left[ \log p(y_{n} \,|\, f^{(2)}, f^{(1)}_{k}, z_{n, k} ) \right]\right)$.
This is the bound presented in~\Cref{sec:iwvi-dgp}.
\vspace*{-20pt}

\section{Importance-Weighted Variational Inference in VAEs}
\label{appsec:iwvae}

Training a variational auto-encoder (\vae,~\citet{KingmaVAE}) with the standard \elbo can lead to generative models with limited generative ability due to the bound being very loose.
One was to improve this is to use a different, tighter, \elbo.
Of particular relevance to our work,~\citet{burda2015iwae,mnih2016variational} propose \textit{importance-weighted} variational inference (\iwvi) to obtain tight variational bounds.
Namely, they use the variational objective
\begin{align}\label{app-eq:iwvi_vae_bound}
    \calL_{\text{IWAE}} \defines 
    \underset{z_{n, 1:K}}{ \E }
    \left[\log \frac{1}{K} \sum_{k=1}^K \frac{p_\theta(x_{n}, z_{n, k})}{q_\phi(z_{n, k} \,|\, x_{n})} \right],
\end{align}
where $z_{n, k} \sim q_\phi(z_{n} \,|\, x_{n})$ and $p_\theta(x_{n}, z_{n, k})$ represents the generative model.
It is worth noting that this bound is as tight as or tighter than the standard \vae \elbo for all $K$, i.e. $\calL_{\text{IWAE}} \geq \calL_{\text{VI-VAE}}$. 
This follows from the fact that the Monte Carlo estimator of the expected importance weight,
\begin{align}\label{app-eq:importance_sampling_estimator}
    \Zhat_{K}^{\text{IWAE}} \defines \frac{1}{K} \sum_{k=1}^K w_{n, k}, \,\, \text{with} \,\, w_{n, k} = \frac{p_\theta(x_{n}, z_{n, k})}{q_\phi(z_{n, k} \,|\, x_{n})},
\end{align}
has decreasing variance for increasing $K$, with a maximum variance at $K=1$.
Moreover, note that for $z_{n, k} \sim q_\phi(z_{n} \,|\, x_{n})$ and $K \rightarrow \infty$, $\Zhat_{K}^{\text{IWAE}}$
converges to the true log marginal likelihood $p_\theta(x_{n})$ as the number of importance samples increases and the bound becomes tight everywhere.

\section{Reparameterization Gradient Estimator of Importance-Weighted VI}
\label{appsec:reparameterization_iwvi-dgp}

We consider the gradient estimator of $\calL_{\textsc{K}}$ with respect to $\phi$, the set of parameters governing the variational distribution over the latent variable $z$, which is given by
\begin{align}\label{app-eq:iwvi_gradient}
\begin{split}
    \gradphi \calL_{\textsc{K}} = \gradphi \sum_{n} \underset{f^{(1)}_{1:K}, z_{n, 1:K}}{ \E } \left[ \log \frac{1}{K} \sum_{k=1}^K \frac{\calF\left(x_{n}, y_{n}, f^{(1)}_{k}, z_{n, k} \right) p(z_{n, k})}{ q(z_{n, k})} \right],
\end{split}
\end{align}
where $z_{n, k}$ are samples from $q(z)$ and $f^{(1)}_{k}$ are samples from the variational distribution $q(f^{(1)})$.
If $z$ and $f^{(1)}$ are reparameterizable---that is samples from $q(z)$ and $q(f^{(1)})$ can be expressed as
\begin{align}
    z_{n, k} = \ztilde(\epsilon^{z}_{k}, \phi) = \mu^{z}(\phi) + \epsilon_k^{z} \odot \sqrt{\Sigma^{z}(\phi)}
\end{align}
and
\begin{align}
    f^{(1)}(x_{n}, z_{n, k}) = \ftilde(\epsilon^{f^{(1)}}_{k}, x_{n}, \ztilde(\epsilon^{z}_{k}, \phi), \psi^{(1)}) = \mu^{f^{(1)}}(x_{n}, \ztilde(\epsilon^{z}_{k}, \phi), \psi^{(1)}) + \epsilon^{f^{(1)}}_{k} \odot \sqrt{\Sigma^{f^{(1)}}(x_{n}, \ztilde(\epsilon^{z}_{k}, \phi), \psi^{(1)})},
\end{align}
respectively, where $\ztilde(\cdot)$ and $\ftilde(\cdot)$ are deterministic functions, $\mu^{f^{(1)}}$ and $\Sigma^{f^{(1)}}$ denote the mean and variance of the predictive distribution over $f^{(1)}$, $\phi$ and $\psi^{(1)}$ are variational parameters of $z$ and $f^{(1)}$ respectively, and $\epsilon^{z}_{k}, \epsilon^{f^{(1)}}_{k} \sim \calN(0, 1)$ are both independent of $\phi$ and $\psi^{(1)}$---then the gradient estimator in~\Cref{app-eq:iwvi_gradient} can be expressed as
\begin{gather}\label{app-eq:iwvi_gradient_reparameterized}
\begin{aligned}
    \gradphi \calL_{\textsc{K}}
    &=
    \gradphi \sum_{n} \underset{\epsilon_{1:K}^{z}, \epsilon^{f^{(1)}}_{1:K}}{ \E } \left[ \log \frac{1}{K} \sum_{k=1}^K \frac{\calF\left(x_{n}, y_{n}, \ftilde(\epsilon^{f^{(1)}}_{k}, \ztilde(\epsilon^{z}_{k}, \phi), \psi^{(1)}) \right) p(\ztilde(\epsilon^{z}_{k}, \phi))}{ q(\ztilde(\epsilon^{z}_{k}, \phi))} \right]
    \\
    &=
    \gradphi \sum_n \iint \prod_{k'=1}^K p(\epsilon_{k'}^{z}) p(\epsilon^{f^{(1)}}_{k'}) \log \frac{1}{K} \sum_{k=1}^K \frac{\calF\left(x_{n}, y_{n}, f^{(1)}_{k}, z_{n, k} \right) p(z_{n, k})}{ q(z_{n, k})} d \epsilon_{1:K}^{z} d \epsilon^{f^{(1)}}_{1:K} \\
    &= \sum_n \iint \prod_{k'=1}^K p(\epsilon_{k'}^{z}) p(\epsilon^{f^{(1)}}_{k'}) \gradphi \log \frac{1}{K} \sum_{k=1}^K \frac{\calF\left(x_{n}, y_{n}, f^{(1)}_{k}, z_{n, k} \right) p(z_{n, k})}{ q(z_{n, k})} d \epsilon_{1:K}^{z} d \epsilon^{f^{(1)}}_{1:K}
    \\
    &=
    \sum_{n} \underset{\epsilon_{1:K}^{z}, \epsilon^{f^{(1)}}_{1:K}}{ \E } \left[ \gradphi  \log \frac{1}{K} \sum_{k=1}^K \frac{\bar{\calF}\left(x_{n}, y_{n}, \ftilde(\epsilon^{f^{(1)}}_{k}, \ztilde(\epsilon^{z}_{k}, \phi), \psi^{(1)}) \right) p(\ztilde(\epsilon^{z}_{k}, \phi))}{ q(\ztilde(\epsilon^{z}_{k}, \phi))} \right]
    \end{aligned}
    \\
	\text{with} \quad \epsilon^{z}_{k} \sim \calN(0, 1) \quad \text{and} \quad \epsilon^{f^{(1)}}_{k} \sim \calN(0, 1), \nonumber
\end{gather}
where we were able to move the gradient inside the integral, since the variational distributions $q_{\phi}(z)$ and $q_{\psi^{(1)}}(f^{\left(1 \right)})$ are reparameterized as described above with the random variables $\epsilon^{z}_{k}$ and $\epsilon^{f^{(1)}}_{k}$ being independent of $\phi$.
Defining the importance weights as
\begin{align}
	w_{n, k}(f^{(1)}, z)
	\defines
	\frac{\calF\left(x_{n}, y_{n}, f^{(1)}_{k}, z_{n, k} \right) p(z_{n, k})}{ q(z_{n, k})},
\end{align}
and having established the dependence of $\calF(\cdot)$ on $x_{n}$, $y_{n}$, $f^{(1)}_{k}$, and $z_{n, k}$, we can now follow~\citet{burda2015iwae} and evaluate the derivative of the logarithm under reparameterization with respect to $\phi$, which yields
\begin{align}\label{app-eq:iwvi_gradient_final}
\begin{split}
	\gradphi \calL_{\textsc{K}} &= \sum_n \underset{\epsilon_{1:K}^{z}, \epsilon^{f^{(1)}}_{1:K}}{ \E } \bigg[ \sum_{k=1}^K \frac{w_{n, k}(\ftilde(\epsilon^{f^{(1)}}_{k}, \ztilde(\epsilon^{z}_{k}, \phi), \psi^{(1)}), \ztilde(\epsilon^{z}_{k}, \phi))}{\sum_{j=1}^K w_{n, j}(\ftilde(\epsilon^{f^{(1)}}_{k}, \ztilde(\epsilon^{z}_{k}, \phi), \psi^{(1)}), \ztilde(\epsilon^{z}_{k}, \phi))}  \gradphi \log w_{n, k}(\ftilde(\epsilon^{f^{(1)}}_{k}, \ztilde(\epsilon^{z}_{k}, \phi), \psi^{(1)}), \ztilde(\epsilon^{z}_{k}, \phi)) \bigg].
\end{split}
\end{align}
\Cref{app-eq:iwvi_gradient_final} is the importance-weighted variational inference gradient estimator for a 2-layer latent \dgp model, which we note has the same functional form as the \iwvae gradient estimator---with the exception of the additional expectation over the non-final layer random function(s) in $\calF\left(x_{n}, y_{n}, f^{(1)}_{k}, z_{n, k} \right)$.

The Monte Carlo gradient estimator is then given by
\begin{align}\label{app:eq:iwvi_gradient_mc}
\begin{split}
    \gradphi \calL_{\textsc{K}} &= \frac{1}{M} \sum_n \sum_{m=1}^M  \sum_{k=1}^K \frac{w_{n, m, k}(\ftilde(\epsilon^{f^{(1)}}_{m, k}, \ztilde(\epsilon^{z}_{m, k}, \phi), \psi^{(1)}), \ztilde(\epsilon^{z}_{m, k}, \phi))}{\sum_{j=1}^K w_{n, m, j}(\ftilde(\epsilon^{f^{(1)}}_{m, k}, \ztilde(\epsilon^{z}_{m, k}, \phi), \psi^{(1)}), \ztilde(\epsilon^{z}_{m, k}, \phi))} \\
	&\qquad \qquad \qquad \qquad \qquad \cdot \gradphi \log w_{n, m, k}(\ftilde(\epsilon^{f^{(1)}}_{m, k}, \ztilde(\epsilon^{z}_{m, k}, \phi), \psi^{(1)}), \ztilde(\epsilon^{z}_{m, k}, \phi)),
\end{split}
\end{align}
where $M$ is the number of Monte Carlo samples used for estimation of $\calL_{\textsc{K}}$.

\section{Deep GP Doubly Reparameterized Gradient Estimator}
\label{appsec:dreg_derivation}

Assuming that reparameterization of $q(z_{n})$ in $\calL_{\textsc{IWAE}}$ is possible, the doubly reparameterized gradient (\dreg) estimator of $\calL_{\textsc{IWAE}}$ at a single data point $x_{n}$ can be expressed as
\begin{align}
\begin{split}
    \widetilde{\Delta}_{n, K}^{\text{IWAE}}(\theta, \phi) \defines \nabla_{\phi} \mathbb{E}_{z_{1: K}}\left[\log \left(\frac{1}{K} \sum_{k=1}^{K} \bar{w}_{n, k}\right)\right]
    &=\nabla_{\phi} \mathbb{E}_{\epsilon_{1: K}}\left[\log \left(\frac{1}{K} \sum_{k=1}^{K} \bar{w}_{n, k}\right)\right] =\mathbb{E}_{\epsilon_{1: K}}\left[\sum_{k=1}^{K} \frac{\bar{w}_{n, k}}{\sum_{j} \bar{w}_{n, j}} \nabla_{\phi} \log \bar{w}_{n, k}\right],
\end{split}
\end{align}
with $\bar{w}_{n, k}=p_{\theta}\left(x_{n}, z_{n, k}\right) / q_{\phi}\left(z_{n, k} | x_{n} \right)$.
Following~\citet{tucker2019dreg} we can rearrange the $\nabla_{\phi} \log \bar{w}_{n, k}$ term and take a Monte Carlo estimate of $\epsilon_{1:K}$ to produce the gradient estimate
\begin{align}\label{app-eq:dreg}
\begin{split}
     &\widetilde{\Delta}_{n, M, K}^{\text{IWAE}}(\theta, \phi) \defines \frac{1}{M} \sum_{m=1}^M \sum_{k=1}^K \frac{\bar{w}_{n, m, k}}{\sum_{k'=1}^K \bar{w}_{n, m, k'}} \left(\frac{\partial \log \bar{w}_{n, m, k}}{\partial z_{n, m, k}} \frac{\partial z_{m, k}}{\partial \phi} - \frac{\partial}{\partial \phi} \log q_\phi(z_{n, m, k} \,|\, x_{n}) \right),
\end{split}
\end{align}
where \mbox{$z_{n, m, k} \stackrel{\text{i.i.d.}}{\sim} q_{\phi}(z_{n} \,|\, x_{n})$} as before.
As stated in~\citet{tucker2019dreg}, the \snr of this gradient estimator scales as $\calO({\sqrt{K}})$ instead of $\calO{(1/\sqrt{K}})$, that is, the \snr improves as $K$ increases.

To derive a \dreg estimator for \dgps, we note that in~\Cref{appsec:reparameterization_iwvi-dgp}, we show that the reparameterization gradient for the gradient with respect to the parameters of the variational distribution over the latent variable $z$ in \iwvi in \dgps (see~\Cref{eq:iwvi_dgp_reparameterization_gradient_estimator}) at a single data point $x_{n}$ can be expressed as 
\begin{gather}\label{app-eq:iwvi_gradient_mc_full}
    \Delta_{n, M, K}^{{\text{DGP}}}(\phi) \defines \frac{1}{M} \sum_{m=1}^M  \sum_{k=1}^K \frac{w_{n, m, k}}{\sum_{j=1}^K w_{n, m, j}} \gradphi \log w_{n, m, k},
    \\
    \text{where} \quad w_{n, m, k} \defines \tilde{\calF}_{n, m, k} \, \frac{p(\ztilde(\epsilon^{z}_{m, k}, \phi))}{ q(\ztilde(\epsilon^{z}_{m, k}, \phi))} \quad \text{and} \quad \tilde{\calF}_{n, m, k} \defines \bar{\calF}(x_{n}, y_{n}, \ftilde(\epsilon^{f^{(1)}}_{m, k}, \ztilde(\epsilon^{z}_{m, k}, \phi), \psi^{(1)})). \nonumber
\end{gather}
Evaluating the $\gradphi \log w_{n, m, k}$ term as above, we get
\begin{align}\label{app-eq:iwvi_gradient_mc_full2}
\begin{split}
    \Delta_{n, M, K}^{{\text{DGP}}}(\phi) &\defines \frac{1}{M} \sum_{m=1}^M  \sum_{k=1}^K \frac{w_{n, m, k}}{\sum_{j=1}^K w_{n, m, j}} \bigg( \frac{\partial \log w_{n, m, k}}{\partial \ztilde} \frac{\partial \ztilde(\epsilon^{z}_{m, k})}{\partial \phi} - \frac{\partial}{\partial \phi} \log q(\ztilde(\epsilon^{z}_{m, k}, \phi)) \bigg),
\end{split}
\end{align}
and note that this estimator strongly resembles the \iwae-\dreg estimator with the difference that here, $\frac{\partial \log w_{n, m, k}}{\partial \ztilde}$ contains the reparameterized \dgp layers.
We exploit this similarity to derive a \dreg estimator for \dgps by following the derivation presented in \citet{tucker2019dreg}.
To do so, we take advantage of the fact that once the reparameterization gradient estimator is established, the derivation in \citet{tucker2019dreg} only relies on the equivalence between the REINFORCE~\citep{williams1992reinforce} and reparameterization gradient estimators and does not rely on the exact expression for $\frac{\partial \log w_{n, m, k}}{\partial \ztilde}$.
This way, we are able to replace the high-variance REINFORCE gradient term, $\frac{\partial}{\partial \phi} \log q(\ztilde(\epsilon^{z}_{m, k}, \phi))$ in~\Cref{app-eq:iwvi_gradient_mc_full} and, obtain a \dreg estimator for \iwvi for \dgps given by
\begin{align}
\label{app-eq:iwvi_dgp_dreg}
\begin{split}
    \widetilde{\Delta}_{n, M, K}^{{\text{DGP}}}(\phi)
    &\defines
    \frac{1}{M} \sum_{m=1}^M  \sum_{k=1}^K \left( \frac{w_{n, m, k}}{\sum_{j=1}^K w_{n, m, j}} \right)^2 \frac{\partial \log w_{n, m, k}}{\partial \ztilde} \frac{\partial \ztilde(\epsilon^{z}_{m, k})}{\partial \phi},
\end{split}
\end{align}
where, unlike for \iwaes,
\begin{align}
    w_{n, m, k} \defines \tilde{\calF}_{n, m, k} \, \frac{p(\ztilde(\epsilon^{z}_{m, k}, \phi))}{ q(\ztilde(\epsilon^{z}_{m, k}, \phi))} \quad \text{with} \quad \bar{\calF}_{n, m, k} \defines \bar{\calF}(x_{n}, y_{n}, \ftilde(\epsilon_{m, k}^{f^{(1)}}, \ztilde(\epsilon^{z}_{m, k}, \phi), \psi^{(1)})).
\end{align}

\section{Experimental Details}
\label{appsec:experimental_details}

\subsection{Datasets}
We design a 1D ``demo'' dataset which exhibits multimodality and is easy to visualize.
It is defined
\begin{align}
    f_1(x) &= \frac{\sin(4x)}{3} + 0.2 e^\epsilon \\
    f_2(x) &= \frac{9 x^2}{30} + 1.5 + 0.2 e^\epsilon \\
    \epsilon &\sim \mathcal{N}(0,1) \\ 
    f(x) &= \begin{cases}
        f_1(x), & \text{with probability } 0.6 \\
        f_2(x),  & \text{otherwise} .
    \end{cases}
\end{align}
\Cref{fig:demo_dataset} shows the demo dataset.
It contains $2000$ points.

\begin{figure}[ht!]
    \centering
    \includegraphics[width=0.65\linewidth, keepaspectratio]{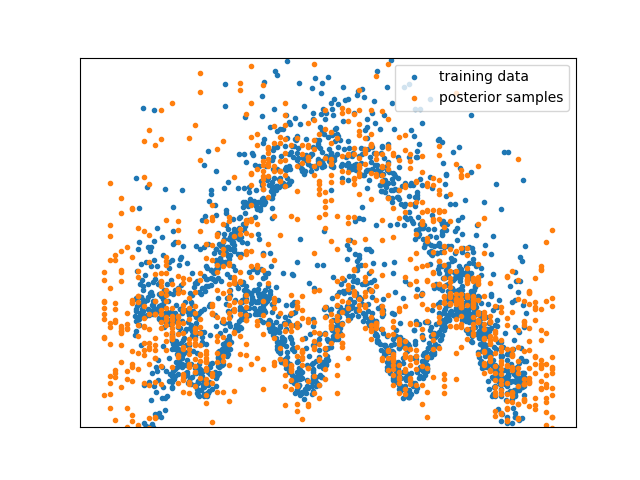}
    \vspace{-15pt}
    \caption{
        1D demo dataset with multimodality. The blue points show the training data, while the orange points are samples from the posterior distribution of a two-layer \dgp fit to the data.
    }
    \label{fig:demo_dataset}
\end{figure}

In~\Cref{fig:snr_reg_example} and~\Cref{fig:gradient-histograms} we consider the demo dataset.

In~\Cref{fig:snr-grid} and~\Cref{fig:learned_vs_fixed_q} we consider a set of real world datasets.
These are the UCI datasets \citep{ucidatasets}.
We reserve 10\% of the dataset as test data, and use the remaining data as training data.
We normalize the data so that the training set has mean zero and standard deviation one. We use the following library to load and preprocess the data, and it gives full details of what steps are performed:
\begin{center}
    \url{https://github.com/hughsalimbeni/bayesian_benchmarks}.
\end{center}

\subsection{Models, Optimization, and Initialization}

We use the same configuration and hyperparameters in all experiments, with the exception of the number of \dgp layers, number of importance samples, batch size, and number of training iterations. We give the configuration of the other hyperparameters below.

\paragraph{Latent Variable} All models considered in the empirical evaluation use latent variables concatenated to the input data.

\paragraph{Parameterization of $q(z)$} We use a fully connected neural network with two hidden layers of $20$ units each, skip connections between each layer, and the $\text{tanh}(\cdot)$ non-linearity. The network has two heads, for the mean and standard deviation of $q(z)$ respectively. Following \citet{salimbeni2019iwvi}, we add a bias of $-3$ to the standard deviation head to ensure that the output is small at initialization, and apply the softplus function to enforce that the output is positive. We initialize the weights using the scheme introduced by \citet{glorot2010understanding}.

The other hyperparameters and initialization schemes match those given by \citet{salimbeni2019iwvi}, which we reproduce below for completeness:

\paragraph{Kernels} RBF with ARD. Lengthscale initialized to the square root of the dimension.

\paragraph{Inducing points} 128 points per \dgp layer. Initialized to the data, if 128 data points or fewer. Otherwise initialized to the cluster centroids returned by kmeans2 from the SciPy package, with 128 points. The first \dgp layer has an additional input dimension due to the latent variable. The inducing points for this dimension are initialized to random draws from a standard normal distribution.

\paragraph{Linear projections between layers} We implement the linear mean functions and multioutput structure using a linear projection of five independent \gps concatenated with their inputs.
We initialized the projection matrix to the first five principle components of the data concatenated with the identity matrix.

\paragraph{Likelihood} We initialize the likelihood variance to $0.01$.

\paragraph{Parameterizations} All positive model parameters are constrained to be positive using the softplus function, clipped to $10^{-6}$. The variational parameters for the sparse \gp layers are parameterized by the mean and the square root of the covariance.

\paragraph{Optimization} For the final \dgp layer we use natural gradients on the natural parameters, with an initial step size of $0.01$. All other parameters are optimized using the Adam optimizer \citep{adamoptimizer} with all parameters set to their TensorFlow default values except the initial step size of $0.005$. We anneal the learning rate of both the Adam and natural gradient steps by a factor of $0.98$ per $1000$ iterations. The mean functions and kernel projection matrices (the P matrices that correlate the outputs) are not optimized.

\subsection{Details by Figure}
\paragraph{\Cref{fig:snr_reg_example}}
\begin{enumerate}
    \item Train a 2-layer model to convergence of the \elbo on the demo dataset.
    \item For a single point randomly chosen from the training set, take $Q=1000$ samples of the gradient estimator. This yields a $Q \times P$ tensor, where $P$ is the number of parameters of the variational distribution over the latent variable, $q(z)$.
    \item Compute the \snr, by taking the mean and standard deviation over the $Q$ dimension.
    \item Plot the \snr for $4$ randomly chosen parameters.
\end{enumerate}

\paragraph{\Cref{fig:gradient-histograms}}
Train a 2-layer model to convergence of the \elbo on the demo dataset. Randomly choose a single data point from the training set and a single parameter of $q(z)$. Sample gradient estimates for this parameter at the data point.

\paragraph{\Cref{fig:snr-grid}}
\begin{enumerate}
    \item Train models of depths $1,2,3,4$ to convergence of the \elbo.
    \item For $10$ randomly chosen data points in the training set:
    \begin{enumerate}
        \item Take $Q=1000$ samples of the gradient estimator yielding a $Q \times P$ tensor, where $P$ is the number of parameters of the variational distribution over the latent variable, $q(z)$.
        \item Compute the \snr, by taking the mean and standard deviation over the $Q$ dimension.
        \item Compute the mean \snr, over the $P$ dimension.
    \end{enumerate}
    \item Take the mean of the \snr over the $10$ points.
\end{enumerate}

\paragraph{\Cref{fig:learned_vs_fixed_q} and~\Cref{tab:reg_dreg_performance}}
We train for $300,000$ iterations.

We estimate the test log-likelihood using Monte Carlo sampling:
\begin{enumerate}
    \item For each point $(x_n y_n)$ in the test set:
    \begin{enumerate}
        \item Take $10000$ samples of $z_{n,k} \sim \calN(z \,|\, 0, I)$, sample $f_{n,k} \sim f(x_n, z_{n,k})$ for each $z_{n,k}$, evaluate the marginal likelihood $\calN(y_n \,|\, f_{n,k}, \sigma^2)$
        \item Take the mean of the sampled marginal likelihood values
    \end{enumerate}
    \item Take the mean of the marginal likelihood over all data points
\end{enumerate}

We find that some train/test splits had significantly worse performance, so we exclude them as outliers.
The criteria we use for this is a test log-likelihood more than two standard deviations from the mean for that dataset and model configuration.
We are careful to exclude the results for both \reg and \dreg, or $q(z)$ and $p(z)$, for a particular split and dataset, even if just one of them was an outlier.

\paragraph{\Cref{fig:reg_vs_fixq_marginals}}
We take the relevant trained models from~\Cref{fig:learned_vs_fixed_q}.
We evaluate each $10000$ times on one randomly selected point from the test set, and plot the resulting $y$ values.

\clearpage

\section{Additional results}
\label{appsec:additional_results}

\begin{figure}[ht]
    \centering
    \begin{subfigure}{0.49\textwidth}
    \includegraphics[width=0.99\linewidth, keepaspectratio]{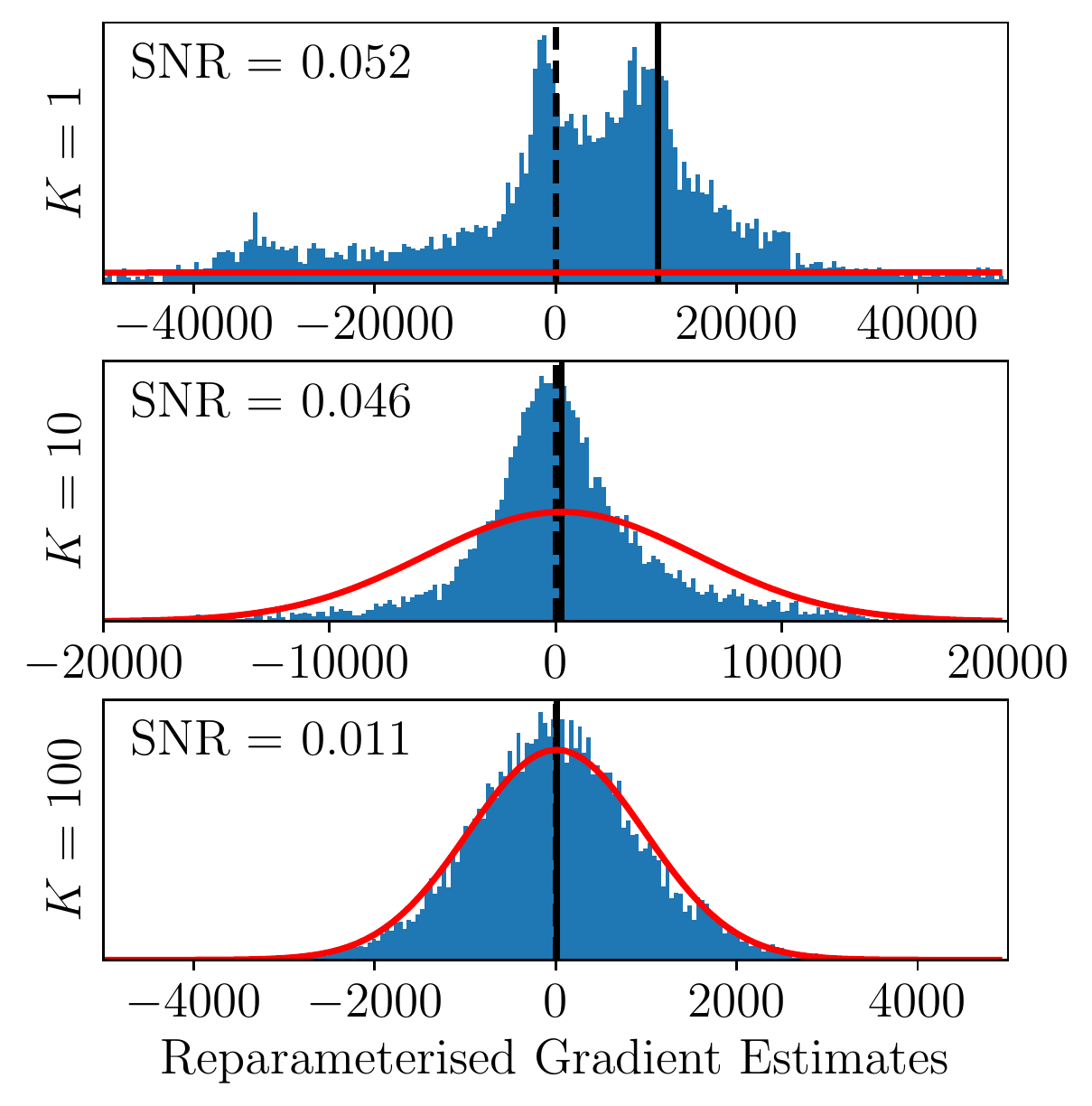}
    \caption{
    Histogram of \reg estimates.
    }
    \label{app-fig:gradient-histograms-dreg-off}
    \end{subfigure}\hfill
    \begin{subfigure}{0.49\textwidth}
    \centering
    \includegraphics[width=0.99\linewidth, keepaspectratio]{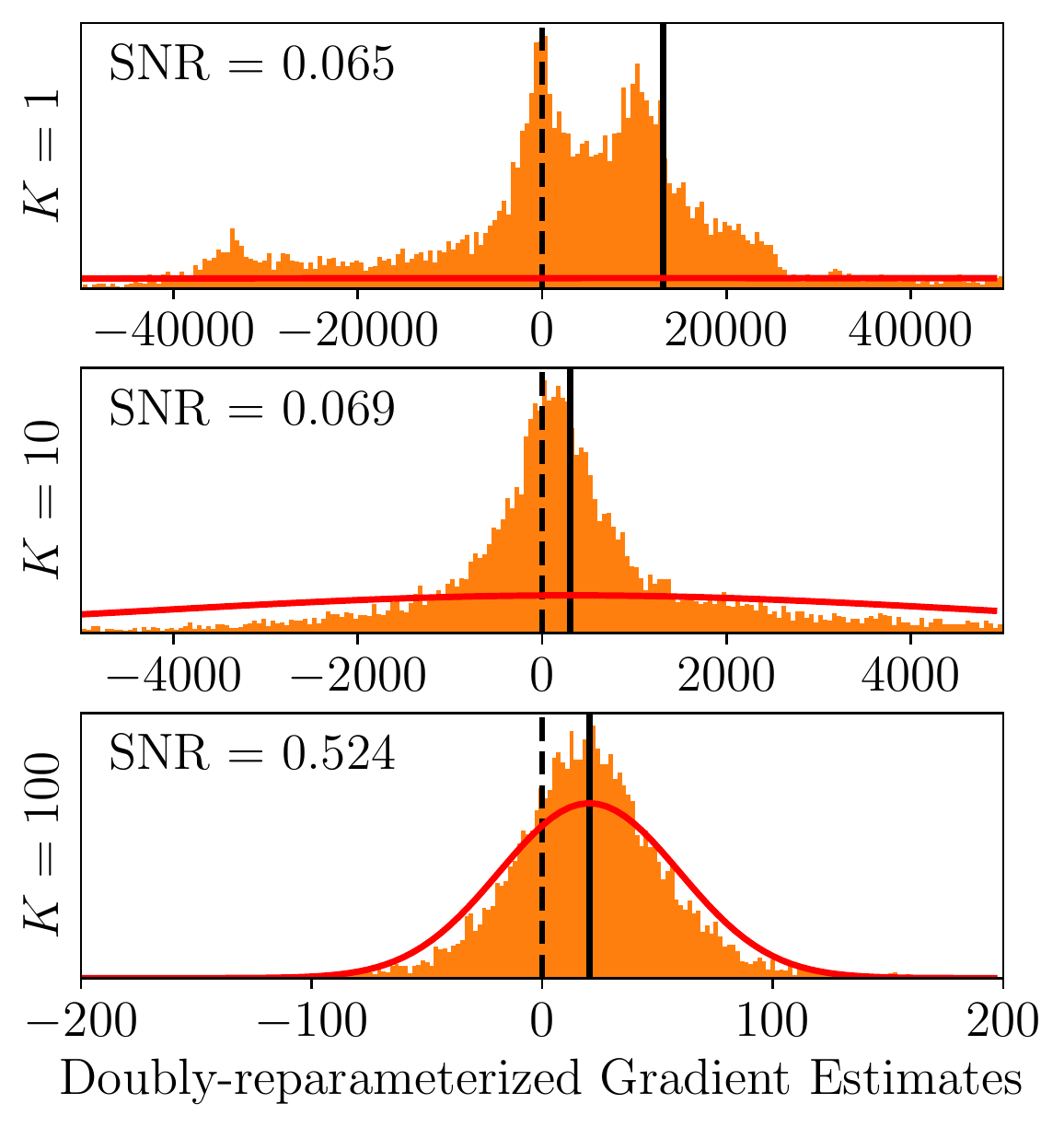}
    \caption{
    Histogram of \dreg estimates.
    }
    \label{app-fig:gradient-histograms-dreg-on}
    \end{subfigure}
    \caption{
    Histograms of \reg (a) and \dreg (b) estimates for varying numbers of importance samples, $K$, with \mbox{$K=1,10,100$}.
    The solid vertical lines denote the mean of the empirical distributions of gradient estimates, and the dashed vertical lines mark zero.
    The red curves denote Gaussian probability densities with the same mean and standard deviation as the empirical distributions.
    Note that for $K=1$ there are a number of very large outliers which we have truncated from the histogram.
    }
    \label{app-fig:gradient-histograms-sep-axes}
\end{figure}

\begin{figure}[ht!]
    \centering
    \includegraphics[width=\linewidth, keepaspectratio]{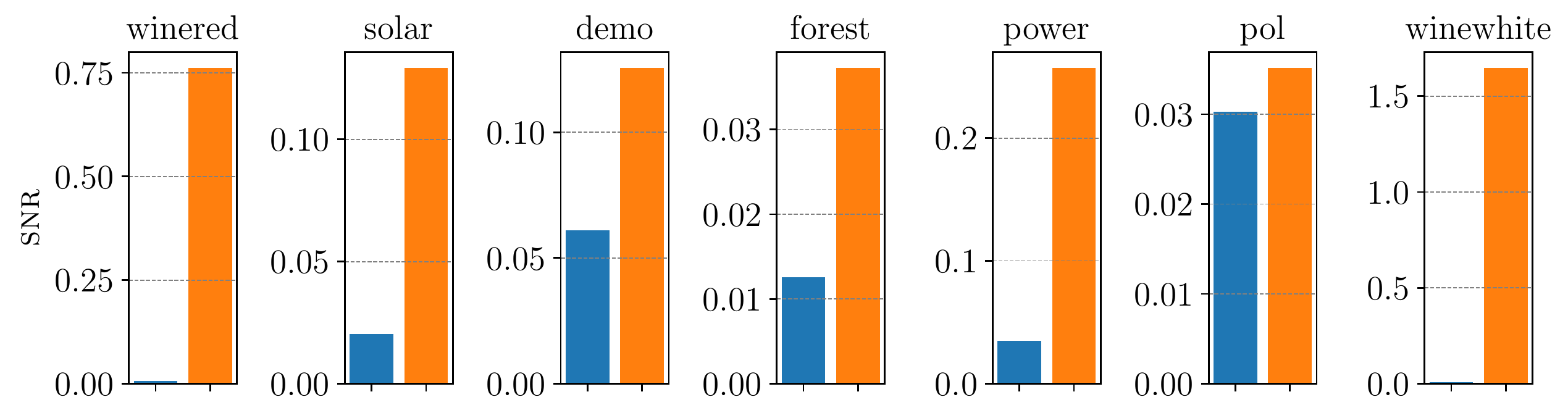}
    \caption{\snr of the \reg (left, blue) and \dreg (right, orange) estimates in the configuration used for~\Cref{tab:reg_dreg_performance}, confirming that the \dreg estimates have a substantially better \snr. The shaded areas show the interquartile range over two seeds.}
    \vspace{-40pt}
\end{figure}

\end{appendices}

\end{document}